\documentclass[lettersize,journal]{IEEEtran}
\usepackage{amsmath,amsfonts}
\usepackage{algorithmic}
\usepackage{algorithm}
\usepackage{array}
\usepackage[caption=false,font=normalsize,labelfont=sf,textfont=sf]{subfig}
\usepackage{textcomp}
\usepackage{stfloats}
\usepackage{url}
\usepackage{verbatim}
\usepackage{graphicx}
\usepackage{cite}
\usepackage{hyperref}
\usepackage{xcolor}

\usepackage{multirow}
\usepackage{booktabs}
\usepackage{amsthm}
\newtheorem{proposition}{Proposition}

\begin{document}
\bstctlcite{IEEEexample:BSTcontrol}
\title{Coefficient Calibration as Selection Pressure in Symbolic Regression}

\author{Mattia~Billa,
        Veronica~Guidetti,
        and~Federica~Mandreoli%
\thanks{We acknowledge the CINECA award under the ISCRA initiative, for the availability of high performance computing resources and support.}
\thanks{M. Billa, V. Guidetti, and F. Mandreoli are with the Department of Physics, Informatics and Mathematics, University of Modena and Reggio Emilia, 41125 Modena, Italy (e-mail: mattia.billa@unimore.it;
veronica.guidetti@unimore.it; federica.mandreoli@unimore.it).}%
}

%



\maketitle

\begin{abstract}
In memetic symbolic regression, candidate structures are compared after coefficient calibration, so the calibration protocol itself contributes to evolutionary selection. Standard centralized calibration evaluates each structure at its pooled-sample optimum, ignoring how stable this calibration is under covariate shifts, and can thus favor structures whose fit relies on sample-specific coefficients.

We propose Dirichlet-Sinkhorn Constant Averaging (DSCA), a calibration strategy that partitions the optimization data into equally sized subsets with different covariate distributions, calibrates each candidate independently on every partition, and evaluates it at the mean of the resulting parameters. We show that the excess loss of DSCA relative to centralized calibration vanishes at the population level for correctly specified, identifiable expressions, whereas under misspecification it persists when partition-specific calibrations do not aggregate to the pooled optimum.

On synthetic benchmarks and ten real-world datasets, DSCA improves functional recovery and the accuracy-complexity trade-off over centralized Broyden–Fletcher–Goldfarb–Shanno and Levenberg-Marquardt calibration, under selection by negative log-likelihood and by the Akaike and Bayesian information criteria. Mechanism analyses associate the DSCA excess loss with the generalization gap and show that the effect is not reproduced by repeated centralized fitting. These results indicate that controlled heterogeneous calibration provides a complementary source of selection pressure in symbolic-regression search.
\end{abstract}


\section{Introduction}
Symbolic regression (SR) aims to discover mathematical models that capture the underlying relationships within data, prioritizing both predictive accuracy and structural interpretability \cite{kronberger2024symbolic, makke2024interpretable}. Unlike traditional regression, which optimizes parameters for a predefined architecture, SR involves a bi-level optimization task. The algorithm must simultaneously search the discrete space of algebraic structures and the continuous space of numerical coefficients. Genetic programming represents the most prominent framework for this task, utilizing evolutionary mechanisms to navigate the combinatorial space of mathematical expressions \cite{de2024srbench++}.

Within this bi-level framework, the identification of optimal numerical constants acts as a critical bottleneck. Historically, genetic programming relied on evolutionary variation to adapt constants, a process that is inefficient and often results in premature convergence or structural bloat \cite{koza1994, sobania2022}. To alleviate this, modern approaches frequently hybridize the evolutionary search with continuous local optimizers, such as the Levenberg-Marquardt or Broyden-Fletcher-Goldfarb-Shanno algorithms \cite{topchy2001,kommenda2020, cranmer2023}.

Accurate coefficient calibration substantially improves the evaluation of a fixed candidate structure, but it also affects the outer structural search. In a memetic SR system, candidate expressions are compared only after their numerical parameters have been adapted to the observed training sample. Consequently, flexible or misspecified structures may become competitive by exploiting sample-specific variation through finely tuned coefficients. This effect is particularly relevant in SR, where most structures visited during evolution are necessarily imperfect approximations of the underlying relationship and coefficient estimation is frequently affected by ill-conditioning \cite{kronberger2022,virgolin2023}.

The resulting generalization-complexity trade-off is typically addressed at the selection level, through explicit parsimony pressure or information- and description-based criteria \cite{soltani2026comparative}. Such mechanisms determine how candidate expressions are compared after their numerical coefficients have already been calibrated on the observed sample. Their penalized likelihood structure usually depends linearly on the number of parameters used by the model, a structure which is not suitable for misspecified models typical of SR and may not be sufficient to guide evolution toward the right accuracy-complexity trade-off \cite{watanabe2010asymptotic,vehtari2017practical}.

Here we investigate a complementary intervention: we modify the coefficient-calibration stage through which an individual receives the parameter values used for its fitness evaluation. The objective is therefore not to improve the training optimum itself, but to make the evaluation of candidate structures less favorable to solutions whose empirical performance relies on highly sample-specific coefficient calibration.

To achieve this, we draw on ideas from distributed and federated parameter estimation, where local models are fitted on distinct data sources and subsequently aggregated into a common parameterization \cite{mcmahan2017}. In those settings, statistical heterogeneity is generally a constraint that the learning procedure must accommodate. In our centralized setting, we instead deliberately induce controlled heterogeneity and exploit it during candidate evaluation. Intuitively, an expression whose good empirical fit depends on finely tuned, sample-specific coefficients should be less compatible with parameter estimates obtained across different views of the covariate space. Aggregating such local estimates can therefore reduce the competitiveness of these candidates during evolutionary selection, providing an additional source of regularization against overfitting that is complementary to the outer model-selection criterion.

\subsection{Our Contributions}

Our contributions are threefold:
\begin{itemize}
\item \textbf{A new calibration strategy for SR.} We introduce Dirichlet-Sinkhorn Constant Averaging (DSCA), which replaces centralized coefficient calibration with calibration over controlled covariate-heterogeneous partitions followed by parameter averaging. Because DSCA modifies only the individual-evaluation stage, it can be readily integrated into standard SR pipelines without changing the underlying search procedure or outer model-selection criterion.

\item \textbf{A theoretical characterization of its selection effect.} 
We characterize the excess loss induced by DSCA relative to centralized calibration and show how this additional calibration-stage penalty can alter evolutionary selection, providing an implicit regularization mechanism. At population level, the excess vanishes for correctly specified and identifiable expressions, whereas under misspecification it can persist when environment-specific calibrations do not aggregate to the pooled optimum. We further characterize its dependence on covariate heterogeneity and its finite-sample behavior.

\item \textbf{An empirical evaluation of effectiveness and mechanism.} Across controlled synthetic benchmarks and ten real-world datasets, we compare DSCA with standard centralized calibration methods under multiple model-selection criteria. DSCA improves functional recovery and the accuracy--complexity trade-off, while mechanism analyses show that its excess loss is associated with the generalization gap and that the observed effect is not reproduced by repeated centralized fitting.
\end{itemize}

The remainder of this article is organized as follows. Section \ref{sec:related} reviews memetic approaches, model selection in SR, and distributed inference. Section \ref{sec:methodology} presents DSCA, its partitioning and averaging steps, and an analysis of the excess loss it induces. Section \ref{sec:setup} describes the experimental setup, and Section \ref{sec:results} reports results on the synthetic and real-world tracks. Finally, Section \ref{sec:conclusions} concludes and outlines directions for future work.

\section{Related Work} \label{sec:related}
\subsection{Memetic Approaches and Parameter Optimization in Symbolic Regression}

Discovering SR models requires navigating two interrelated search spaces. The algorithm must identify an appropriate algebraic structure and simultaneously determine the optimal numerical coefficients for that structure. In traditional genetic programming, parameters were often initialized as ephemeral random constants that remained fixed throughout the evolutionary process \cite{koza1994, kommenda2013, kommenda2013b, sobania2022}. This paradigm limited the precision of the resulting models, as the algorithm was forced to rely on the stochastic combination of existing constants via crossover to approximate optimal values.

To overcome these structural limitations, researchers introduced memetic algorithms, which hybridize population-based evolutionary search with localized continuous optimization strategies \cite{moscato1989, trujillo2018}. Topchy and Punch provided foundational contributions to this area by demonstrating that embedding a local gradient descent search for numeric leaf values significantly accelerated the convergence of genetic programming \cite{topchy2001}. Their work also revealed a critical secondary effect: the integration of local search inherently steered the global evolutionary process toward algebraic structures that were mathematically more amenable to parametric tuning.

Subsequent advancements have formalized the application of continuous optimization within discrete expression trees. Kommenda and colleagues conducted extensive analyses on the use of nonlinear least squares techniques, specifically employing the Levenberg-Marquardt algorithm \cite{lm1963, kommenda2013, kommenda2013b, kommenda2020}. To supply the optimizer with exact gradient information, they incorporated automatic differentiation, a methodology that proved faster and more accurate than standard numerical or symbolic differentiation. This approach has become a core component in high-performance genetic programming frameworks such as Operon, which relies on exact function derivatives for local parameter identification \cite{burlacu2020, kronberger2022}. Other recent approaches, such as PySR \cite{cranmer2023}, formalize this process into a dedicated evolve, simplify, and optimize cycle that periodically refines constants via algorithms like BFGS \cite{broyden1970convergence}. 
Concurrently, researchers have tackled this problem using Bayesian frameworks to handle both model selection and uncertainty estimation \cite{bomarito2022, bomarito2023, billa2025}.

Recent comparative studies have evaluated a wider array of continuous optimizers in the context of SR. For example, dos Reis et al. evaluated multiple constant optimization methods, noting that no single technique is universally optimal across all problem types \cite{reis2024benchmarking}. 
Furthermore, Harrison et al.  systematically compared first- and second-order gradient-based coefficient optimization in SR and found second-order approaches to be consistently more effective, despite their higher per-step computational cost \cite{virgolin2023}. This distinction is particularly relevant in SR, where individual expressions generally contain relatively few numerical parameters and nonlinear least-squares or quasi-Newton methods are therefore computationally viable. Recent benchmarking further shows that the quality of coefficient optimization can itself affect the structures favored by evolution: poorly calibrated candidates may remain artificially short because introducing additional constants provides little fitness advantage when those constants cannot be optimized effectively \cite{reis2024benchmarking}.

However, integrating continuous optimization into genetic programming presents algorithmic difficulties. Kronberger highlighted that the evolutionary process frequently generates overparameterized expressions containing mathematically redundant constants \cite{kronberger2022}. This overparameterization leads to rank-deficient and ill-conditioned Jacobian matrices during the nonlinear least squares optimization process, which can degrade optimizer efficiency and reliability.
These observations emphasize that coefficient calibration and structural search cannot be regarded as independent components: the quality and stability of the inner optimization directly influence which expressions remain competitive during evolution.

\subsection{Model Selection and Complexity Control in Symbolic Regression}

Effective coefficient calibration does not by itself resolve the model-selection problem. Improving the fit of the numerical parameters makes a given expression more competitive, but does not determine whether its structure provides a useful approximation of the underlying relationship. Once their coefficients have been calibrated, symbolic expressions must therefore still be compared according to both empirical fit and structural complexity.

Genetic programming has traditionally addressed this problem through explicit parsimony pressure and multi-objective accuracy-complexity optimization, while more recent work has investigated likelihood- and description-based selection criteria such as AIC, AICc, BIC, and MDL \cite{schwarz1978estimating,kass1995bayes,soltani2026comparative,kronberger2026guiding}. These criteria act at the outer selection stage by correcting empirical fit according to some notion of model complexity. AIC and BIC use the number of fitted parameters through their familiar $2p$ and $p\log N$ corrections, whereas recent MDL formulations for SR additionally account for structural complexity and coefficient precision through an evidence-related description-length objective.

The standard forms of these corrections nevertheless rely on assumptions that are not automatically satisfied by symbolic models. In particular, parameter-count penalties such as those underlying AIC and BIC are derived under regular parametric asymptotics, including local identifiability and nonsingular Fisher information. Symbolic expressions can violate these conditions through redundant coefficients, equivalent parameterizations, inactive components, and highly correlated or weakly identified parameter directions \cite{kronberger2022}. In singular models, nominal parameter count no longer provides an adequate measure of effective statistical complexity; for evidence-based quantities, the standard Laplace expansion is instead replaced by geometry-dependent corrections such as those involving the real log canonical threshold \cite{watanabe2009algebraic,watanabe2013widely}.

Misspecification raises a distinct problem. The model-identification and evidence-based interpretations underlying criteria such as BIC and Bayesian or description-length model comparison are most naturally formulated in an \emph{M-closed} setting, in which the data-generating distribution is represented within the candidate model class \cite{bernardo1994bayesian}. Symbolic-regression search is instead naturally \emph{M-open}: most individuals encountered during evolution are misspecified approximations, and the true relationship need not belong to the symbolic hypothesis space, let alone to the finite population available at a particular generation. AIC is less directly tied to this assumption because its target is predictive risk, but its standard parameter-count correction still relies on regular asymptotics.

Despite these limitations, selection criteria have proved useful in SR, particularly for controlling formula bloat and favoring parsimonious expressions. AIC and BIC have been successfully used to guide evolutionary selection, while recent experiments with richer MDL-based penalties show that their stronger complexity pressure can also lead to premature convergence or excessively simple expressions when used directly as evolutionary fitness \cite{soltani2026comparative,kronberger2026guiding}. More generally, these criteria have been more directly effective at controlling the complexity of the formulas retained by the search than at preventing misspecified, sample-sensitive structures from becoming competitive after coefficient calibration. DSCA addresses this complementary mechanism: rather than adding another outer complexity penalty, it modifies the coefficient calibration through which an individual is evaluated and can therefore be combined with AIC, BIC, MDL, or other selection criteria.

\subsection{Distributed Inference, Federated Learning, and Subset-Based Model Aggregation}

Partitioning large datasets to compute local estimators independently and subsequently aggregating the results has emerged as a highly scalable paradigm for statistical learning. Theoretical analyses demonstrate that averaging local solutions, such as empirical risk minimizers, can achieve the minimax optimal statistical convergence rates of a centralized model, provided the number of data partitions is bounded appropriately \cite{zhang2013, zhang2013comm, rosenblatt2016optimality}. The central idea is that parameter information extracted independently from different subsets can be combined into a single global estimate, provided that the local estimators are sufficiently compatible.

This mathematical architecture forms the foundation of federated learning, where the objective is to collaboratively train a global model across decentralized clients while strictly preserving data privacy \cite{mcmahan2017}. A critical challenge in federated networks is the presence of non-identically distributed data across participating clients \cite{hsu2019measuring}. To benchmark algorithmic resilience to statistical heterogeneity, researchers often synthesize non-ideal partitions using Dirichlet distributions. The concentration parameter of the Dirichlet distribution provides a convenient way to control the degree of statistical heterogeneity of feature and label among simulated clients.

Translating these principles to SR, early federated frameworks aggregated fitness across clients but lacked global parameter estimation, restricting models to evaluating randomly sampled formulas with fixed ephemeral random constants \cite{dong2022}. To resolve this fundamental limitation, Billa et al. developed Bayesian Federated Symbolic Regression, which fuses local Gaussian posteriors for marginal likelihood-based model selection without sharing raw data \cite{billa2025}. DSCA starts from the same general idea of estimating numerical parameters locally and combining information across heterogeneous data sources, but uses it for a different purpose. In a genuine federated problem, heterogeneity is imposed by the participating clients and aggregation seeks to recover an effective common model despite that heterogeneity. In our centralized setting, the partitions can instead be designed deliberately. We exploit this freedom to construct subsets with different covariate marginals while preserving the conditional distribution of the response, and use arithmetic aggregation not as a communication-efficient approximation to centralized optimization, but as a controlled perturbation of candidate evaluation during evolutionary selection.

\section{Methodology} \label{sec:methodology}

The proposed approach, termed Dirichlet-Sinkhorn Constant Averaging (DSCA), modifies the coefficient-calibration stage through which candidate expressions are evaluated during the evolutionary search.

Let $\mathcal{D}=\{(x_i,y_i)\}_{i=1}^{N}$ denote the input dataset. Standard memetic calibration fits the coefficients of each candidate once on the available training data. DSCA replaces this single calibration step with two stages:

\begin{itemize}
 \item Once, before the search, the dataset $\mathcal{D}$ is used to create $K$ equally sized subsets with controlled covariate heterogeneity.
 \item During the search, every candidate is calibrated independently on each subset and the resulting parameter vectors are averaged to obtain the parameterization used for fitness evaluation.
\end{itemize}

The following subsections describe these two contributions in detail.

\subsection{Surrogate-Based Dirichlet-Sinkhorn Partitioning}
Starting from $\mathcal{D}$, the objective is to generate $K$ mutually exclusive, equally sized subsets $\mathcal{D}_1, \dots, \mathcal{D}_K$. 

Our mechanism is inspired by a recent work where Dirichlet allocation is used to simulate client heterogeneity in federated SR \cite{billa2025}. Here, we use a similar approach in a centralized setting to leverage controlled data heterogeneity as a probe of model calibration. Notably, unlike \cite{billa2025}, sample allocation is not driven by target observation quantiles, but is a function of covariates only. As we will show in Section~\ref{sec:theory}, this ensures that $p(y\mid x)$ stays identical across subsets. This property guarantees that the population penalty vanishes for correctly specified
and identifiable models, while misspecified models can incur a persistent penalty that,
in the near-homogeneous perturbative regime, scales with partition heterogeneity.

\begin{algorithm}[t]
\caption{Surrogate-Based Dirichlet-Sinkhorn Partitioning}
\label{alg:dirichlet_sinkhorn}
\begin{algorithmic}[1]
\REQUIRE Training set $\mathcal{D}$ of size $N$, surrogate fraction $\rho$, partitions $K$, strata $S$, heterogeneity
$\alpha$
\STATE Draw $\mathcal{D}_{\text{sur}} \subset \mathcal{D}$ with $|\mathcal{D}_{\text{sur}}| = \lfloor \rho N \rfloor$;
set $\mathcal{D}_{\text{opt}} = \mathcal{D} \setminus \mathcal{D}_{\text{sur}}$, $N' = |\mathcal{D}_{\text{opt}}|$.
\STATE Fit ridge surrogate $g$ on $\mathcal{D}_{\text{sur}}$ (penalty by LOO CV).
\STATE Score $u_i \leftarrow g(x_i)$ for all $i \in \mathcal{D}_{\text{opt}}$.
\STATE Sort $\mathcal{D}_{\text{opt}}$ by $u$ and bin into $S$ equal strata.
\STATE Initialize $M \in \mathbb{R}^{S \times K}$ where each row $M_{s,*} \sim \text{Dir}(\alpha \mathbf{1}_K)$.
\STATE Define row targets $\mathbf{r} = (N'/S)\mathbf{1}_S$ and column targets $\mathbf{c} = (N'/K)\mathbf{1}_K$.
\FOR{iteration $1, 2, \dots, 50$}
    \STATE Scale rows: $M_{s,k} \leftarrow M_{s,k} \times (r_s / \sum_{j=1}^K M_{s,j})$
    \STATE Scale columns: $M_{s,k} \leftarrow M_{s,k} \times (c_k / \sum_{i=1}^S M_{i,k})$
\ENDFOR
\STATE Initialize $M_{\text{int}} = \lfloor M \rfloor$.
\STATE Row deficits $d^r_s = r_s - \sum_{j=1}^K M_{\text{int}_{s,j}}$; column deficits
$d^c_k = c_k - \sum_{i=1}^S M_{\text{int}_{i,k}}$.
\STATE Fractional remainders $F = M - M_{\text{int}}$.
\FOR{$s = 1$ to $S$}
    \WHILE{$d^r_s > 0$}
        \STATE Select column $k^*$ with $d^c_{k^*} > 0$ maximizing $F_{s,k^*}$.
        \STATE $M_{\text{int}_{s,k^*}} \leftarrow M_{\text{int}_{s,k^*}} + 1$
        \STATE $d^r_s \leftarrow d^r_s - 1$; \quad $d^c_{k^*} \leftarrow d^c_{k^*} - 1$
        \STATE $F_{s,k^*} \leftarrow -1$ \COMMENT{Prevent immediate re-selection}
    \ENDWHILE
\ENDFOR
\STATE Group $\mathcal{D}_{\text{opt}}$ by strata and shuffle indices.
\STATE Slice and deal to $K$ partitions according to $M_{\text{int}}$.
\RETURN $K$ partitions and the evaluation set $\mathcal{D}_{\text{opt}}$.
\end{algorithmic}
\end{algorithm}

Given $\mathcal{D}$, the partitioning procedure is therefore guided by four main parameters: i) the fraction $\rho$ of $\mathcal{D}$ which is used to define the covariate function driving sample stratification for allocation, ii) the number of subsets  $K$, iii) the number of strata $S$, and iv) the concentration parameter, $\alpha$, which controls data heterogeneity across partitions. Algorithm \ref{alg:dirichlet_sinkhorn} details the procedure, which is organized into six main stages:

\begin{enumerate}
    \item \textbf{Surrogate Scoring (steps 1-3):} A fraction $\rho$ of $\mathcal{D}$ is set aside as $\mathcal{D}_{\text{sur}}$ and used to fit a ridge regression $g$, predicting $y$ from $x$, with the penalty selected by leave-one-out cross-validation. 
    The subset $\mathcal{D}_{\text{sur}}$ is then discarded. The remaining $N' = \lceil (1-\rho)N \rceil$ samples form $\mathcal{D}_{\text{opt}}$, which will be partitioned into $\mathcal{D}_1,\ldots,\mathcal{D}_K$. To this aim, $x_i\in \mathcal{D}_{\text{opt}}$ receives the score $u_i = g(x_i)$.
    
    The surrogate scoring function $g$ is a deterministic function of the covariates alone, and fitting $g$ on held-out data keeps it independent of the targets later used for calibration. Surrogate quality affects how informative the induced heterogeneity is, not the validity of the argument (see Section \ref{sec:theory}): a poor surrogate simply yields nearly random, weakly heterogeneous partitions. Any covariate-measurable score can be used, including a single input variable correlated with the target or the output of a shallow tree ensemble.

    \item \textbf{Stratification (step 4):} $\mathcal{D}_{\text{opt}}$ is sorted by $u$ and divided into $S$ equal-sized strata of $N'/S$ samples each. 

    \item \textbf{Dirichlet Sampling (step 5):} An $S \times K$ allocation matrix $M$ is initialized to define the allocation of each stratum across partitions. Each row is drawn from a symmetric Dirichlet distribution, $M_{s,*} \sim \text{Dirichlet}(\alpha \mathbf{1}_K)$, so that $M_{s,k}$ is the fraction of stratum $s$ tentatively assigned to partition $k$, and each row describes how stratum $s$ is spread over the $K$ partitions. 
    
    Large $\alpha$ spreads each stratum more evenly across partitions, so that the $\mathcal{D}_k$ become increasingly similar in covariate composition to $\mathcal{D}_{\mathrm{opt}}$; small $\alpha$ concentrates each stratum in fewer partitions and produces stronger covariate heterogeneity.

    \item \textbf{Matrix Balancing (steps 6-10):} $M$ is subjected to the Sinkhorn-Knopp algorithm \cite{sinkhorn1967concerning} to enforce the marginal constraints:  Rows $M_{s,*}$ must sum to $N'/S$ so that every sample is used, and columns $M_{*,k}$ to $N'/K$ so that partitions are equally sized. The matrix is updated by alternating row and column scaling:
    \begin{equation}
    M_{s,k} \leftarrow M_{s,k} \frac{N'/S}{\sum_{j=1}^K M_{s,j}} \quad M_{s,k} \leftarrow M_{s,k} \frac{N'/K}{\sum_{i=1}^S M_{i,k}}
\end{equation}
    This iterative proportional fitting converges to a balanced matrix of fractional allocations, typically in fewer than ten iterations.

    \item \textbf{Integer Rounding (steps 11-21):} The balanced matrix is converted to integers by $M_{\text{int}} = \lfloor M \rfloor$, and the residual samples are greedily assigned by incrementing the cells with the largest fractional remainders until the integer margins match the targets exactly.

    \item \textbf{Slicing (steps 22-23):} Sample indices are grouped by stratum, shuffled, and sliced mutually exclusively according to $M_{\text{int}}$.
\end{enumerate}

Since partitioning is performed once, before the evolutionary search begins, and the surrogate fit is a single ridge regression on at most $\rho N$ points, the overhead is negligible relative to the search itself.

\subsection{Local Optimization and Averaging}

Before numerical calibration, each candidate expression $f(\cdot,\theta)$ is simplified symbolically. This preprocessing removes exact algebraic redundancies and structurally non-identifiable parameterizations whenever they can be detected, thereby reducing trivial non-identifiabilities before parameter vectors are optimized and averaged. For example, a parameterization such as $\theta_1\theta_2 e^{\theta_3}h(x) $ depends only on a single effective multiplicative coefficient and should therefore be collapsed before numerical optimization to, e.g., $\theta_1  h(x) $.

During the evolutionary search, $f(\cdot,\theta)$ is calibrated independently on each of the $K$ generated partitions. To this end, we apply a gradient-based optimizer, here BFGS \cite{broyden1970convergence}, to each individual partition, yielding $K$ locally calibrated parameter vectors $\hat{\theta}_1,\ldots,\hat{\theta}_K$, where $\hat{\theta}_k$ denotes the numerical solution returned on partition $k$.

The $K$ sets of local constants are aggregated by computing their arithmetic mean. This produces a single, averaged set of constants for the candidate expression:
\begin{equation}
    \bar{\theta} = \frac{1}{K} \sum_{k=1}^K \hat{\theta}_k\,.
\end{equation}
Finally, the fitness of the candidate is computed on $\mathcal{D}_{\text{opt}}$, the union of the $K$ partitions, fixing parameter calibration to $\bar{\theta}$.

\subsection{Theoretical analysis of DSCA implications}\label{sec:theory}
Replacing centralized calibration with the average of $K$ locally calibrated
parameter vectors changes the loss assigned to a candidate expression. We analyze
this effect in five steps. First, we isolate the excess loss that is directly relevant
to evolutionary selection. We then formalize the partitioned data as covariate-reweighted
population environments and characterize this excess under correct specification and
misspecification. Next, we connect the population result with the empirical
finite-sample behavior. We then study how the surrogate scoring function $g$ and the heterogeneity
parameter $\alpha$ determine whether, and how strongly, the mechanism acts. Finally,
we discuss how the assumptions of the theoretical analysis relate to the numerical
implementation. Proofs and full derivations are given in the supplementary material.

Throughout this section, we use the following notation. Let $(X,Y)$ denote an input-target variables and $f(X,\theta)$ a candidate symbolic expression with parameter vector $\theta$. We interpret $f(X,\theta)$ as a model for the conditional mean of $Y$ given $X$.
Equivalently, under additive Gaussian noise, $Y=f(X,\theta)+\varepsilon$, where $\varepsilon\sim\mathcal{N}(0,\sigma^2)$, thus the candidate induces the conditional model
$p_\theta(y\mid x)=\mathcal{N}(f(x,\theta),\sigma^2)$. For fixed noise variance, maximizing this likelihood with respect to $\theta$ is equivalent to minimizing mean squared error. We therefore describe coefficient calibration through its empirical loss: $L_k(\theta)$ denotes the mean loss on partition $\mathcal{D}_k$, while $L_{N'}(\theta)$ denotes the corresponding loss on the complete optimization set $\mathcal{D}_{\mathrm{opt}}$. Their population counterparts, introduced below, are denoted by $\Lambda_k(\theta)$ and $\Lambda(\theta)$.

\subsubsection{DSCA and selection-relevant excess loss}
Because the partitions are disjoint and equally sized,
\begin{equation}\label{eq:mixture}
    L_{N'}(\theta)=\frac{1}{K}\sum_{k=1}^{K}L_k(\theta)
\end{equation}
holds exactly for every $\theta$. Let  $\hat\theta_k=\arg\min_\theta L_k(\theta)$ be the local empirical optima, $\bar\theta=K^{-1}\sum_k\hat\theta_k$ their average, and $\hat\theta_{N'}=\arg\min_\theta L_{N'}(\theta)$ the centralized optimum. 

DSCA evaluates the candidate at $\bar\theta$, whereas centralized calibration
would evaluate the same structure at $\hat\theta_{N'}$. The selection-relevant
excess loss introduced by DSCA is therefore
\begin{equation}\label{eq:delta}
    \Delta_{N'} :=  L_{N'}(\bar\theta) -  L_{N'}(\hat\theta_{N'}) .
\end{equation}
By construction, $\Delta_{N'}\geq0$ under exact minimization. It measures how much
additional empirical loss is assigned to a fixed candidate structure because DSCA
uses the average of the locally calibrated parameters rather than the centralized
optimum. This is the quantity through which DSCA can modify evolutionary selection.

\subsubsection{Covariate-induced environments and population behavior}
DSCA operates empirically on the $K$ finite partitions $\mathcal{D}_1,\ldots,\mathcal{D}_K$. To determine whether the resulting excess
$\Delta_{N'}$ is purely a finite-sample effect or persists at population level, we
associate each finite partition $\mathcal{D}_k$ with a population environment
$P_k$ induced by the same covariate-reweighting mechanism.

Let $P$ denote the original population law of $(X,Y)$, with density
$p(x,y)=p(y\mid x)p(x)$.  Given the fitted surrogate $g$, let
$A\in\{1,\ldots,K\}$ denote the population environment assigned to an observation
and define $\pi_k(u):=\Pr(A=k\mid g(x)=u)$. Since assignment depends on the observation only through $g(X)$, $\Pr(A=k\mid X=x,Y=y)=\pi_k(g(x))$. Equal-size balancing gives
$\Pr(A=k)=1/K$, and hence environment $k$ has density $p_k(x,y) = p(x,y\mid A=k) = K\,\pi_k(g(x))\,p(x,y)$.
Consequently, $K^{-1}\sum_k p_k(x,y)=p(x,y)$ and, more importantly,
\begin{equation}\label{eq:cond-invariance}
    p_k(y\mid x)=p(y\mid x)\qquad\text{for every }k\,.
\end{equation}
Thus the environments may differ in their covariate marginals while preserving
the same conditional law $p(y\mid x)$. This requires only that, once fitted, the
allocation score be a function of $x$; the quality of the surrogate allocation function $g$ affects the informativeness
of the induced heterogeneity, not the validity of
\eqref{eq:cond-invariance}.

The population counterparts of $L_k$ and $L_{N'}$ are $\Lambda_k(\theta):=\mathbb{E}_{P_k}
\!\left[(Y-f(X,\theta))^2\right],$ and $\Lambda(\theta):=\mathbb{E}_{P}\!\left[(Y-f(X,\theta))^2\right]$, with minimizers $\theta_k^*$ and $\theta^*$, respectively. Under the environment construction conditions, the population counterpart of $\Delta_{N'}$ is then $\Delta:=\Lambda(\bar\theta^*)-\Lambda(\theta^*)$, where $\bar\theta^*=K^{-1}\sum_k\theta_k^*$. We now inspect, under identifiability assumptions, how $\Delta$ changes between well specified and misspecified models. 

\noindent\textbf{Correctly specified models.} A candidate is correctly specified for squared-error calibration if there exists a parameter vector $\theta^*$ such that $f(X,\theta^*)=\mathbb{E}[Y\mid X]$ almost surely. Because the induced environments
modify only the covariate marginal while preserving $p(y|x)$, the same optimum is retained in every environment. Under regularity and support conditions (details can be found in the supplementary material), $\theta_k^*=\theta^*$ for all $k$ and $\bar\theta^*=\theta^*$, leading to $\Delta=0$.
This is the main reason why we define allocation based on $x$ only: target-dependent partitioning can alter $p(y\mid x)$ and therefore move the optimum even for a correctly specified candidate.

\noindent\textbf{Misspecified models.}
A candidate is misspecified when no parameter vector $\theta$ reproduces the
conditional mean $\mathbb{E}[Y\mid X]$. Its environment-specific optimum
$\theta_k^*$ is then a covariate-weighted pseudo-true parameter. Reweighting
different regions of covariate space can change the approximation trade-off made by
the candidate and therefore produce different preferred parameter values across
environments. Under a local quadratic approximation, the supplementary material shows that
\begin{equation}\label{eq:quadratic-delta}
    \Delta=
    \frac{1}{2}
    \left\|
    \frac{1}{K}\sum_{k=1}^{K}
    (H_k-\bar H)
    (\theta_k^*-\bar\theta^*)
    \right\|_{\bar H^{-1}}^{2},
\end{equation}
where $H_k=\nabla^2\Lambda_k(\theta_k^*)$ is the local curvature in environment
$k$, $\bar H=K^{-1}\sum_k H_k$ is the average curvature, and
$\|v\|_{\bar H^{-1}}^2=v^\top\bar H^{-1}v$.
Thus $\Delta$ depends jointly on how the preferred parameters shift across
environments and on how their local curvatures differ. This shows that misspecification alone is not sufficient to always ensure a population-level DSCA penalty. If all $H_k$ coincide or if by chance the exact cancellation $\sum_{k=1}^{K}(H_k-\bar H)d_k=0$ holds, then $\Delta=0$ even when the $\theta_k^*$ differ substantially. Absent such a symmetry or compensating structure, heterogeneous curvatures and displacements generically produce a nonzero imbalance, so that $\bar\theta^*\neq\theta^*$ and a persistent population excess arises.

Overall, the population effect characterized above is governed by specification,
identifiability, and sensitivity of the preferred calibration to the induced
environments, rather than by parameter count itself. DSCA is therefore not an explicit
complexity penalty: candidates with many coefficients need not incur a population
penalty, while misspecified candidates can incur one when heterogeneous calibration
moves the averaged parameters away from the pooled optimum. Structural complexity
must consequently remain controlled by the outer selection mechanism.

\subsubsection{Finite-sample behavior}
The preceding results describe the population effect, whereas DSCA operates on
finite empirical partitions. We therefore ask whether the observed excess
$\Delta_{N'}$ is only a sampling discrepancy or instead converges to the
population quantity $\Delta$. Under standard regularity conditions and assuming identifiable models, fixed $K$,
and large $N'$,
\begin{equation}\label{eq:dichotomy}
\Delta_{N'}=
\begin{cases}
O_p(1/N'), & \text{correctly specified},\\[2pt]
\Delta+O_p(1/\sqrt{N'}), & \text{misspecified };
\end{cases}
\end{equation}
where $O_p(\cdot)$ denotes stochastic order, i.e., $X_{N'}=O_p(a_{N'})$ means that $X_{N'}/a_{N'}$ remains bounded in probability. Thus, under correct specification the entire DSCA excess is a second-order finite-sample effect, whereas when $\Delta>0$ a nonzero population excess persists and only its sampling fluctuations vanish, at a slower rate.

\subsubsection{The role of environment definition and heterogeneity}
The DSCA effect depends on two complementary aspects of environment construction:
the surrogate allocation function $g$ determines which regions of covariate space are contrasted, whereas
$\alpha$ controls the strength of the resulting heterogeneity.

\noindent\textbf{Surrogate informativeness.} For a misspecified candidate, an informative surrogate should induce environments
that favor different calibrations. Under squared-error loss, the relevant
interaction for environment $k$ is
\begin{equation}
    r_k:=\mathbb{E}_{P}\!\left[
        \delta_k(X)
        \nabla_\theta
        \mathbb{E}\!\left[
            (Y-f(X,\theta))^2 \mid X
        \right]_{\theta=\theta^*}
    \right],
\end{equation}
where $\delta_k(X)=\pi_k(g(X))-1/K$ is the covariate reweighting induced by the
surrogate. Thus $r_k$ measures whether that reweighting is aligned with regions in
which the candidate prefers a different calibration: small $r_k$ implies little
first-order displacement of the local optimum, whereas nonzero $r_k$ induces a
calibration shift. The supplementary material formalizes this relation perturbatively.

Choosing $g$ as a predictor of $Y$ provides a candidate-independent,
response-informed ordering of covariate space, making it more likely to contrast
regions with different approximation trade-offs/estimated noise levels. It is not optimal for every candidate, since the relevant response also depends on candidate-specific parameter sensitivity, but it provides a common response-informed
partitioning rule that can be constructed once before the search.

\noindent\textbf{Heterogeneity level.}
Within DSCA, $\alpha$ is the concentration parameter
of the initial Dirichlet allocation and therefore controls the strength of the induced covariate heterogeneity.
As $\alpha\to\infty$, the allocation probabilities approach $1/K$, the environments become homogeneous, $P_k\to P$, and $\Delta$ vanishes. More precisely, the perturbation analysis in the supplementary material gives
\[
    \Delta=O\!\left((K\alpha+1)^{-2}\right)
    \qquad \text{as }\alpha\to\infty.
\]
Hence nonzero heterogeneity is necessary for DSCA to induce a population-level selection effect, although the result does not imply that $\Delta$ is monotone in $\alpha$ or identify an optimal heterogeneity level.

Heterogeneity also cannot be increased indefinitely, since excessively concentrated environments may provide insufficient covariate support for local identification. Parameter variation can then reflect weak identification or induced near-singularity rather than genuine sensitivity to covariate reweighting.

\noindent\textbf{From theoretical assumptions to practical implementation.}
The preceding analysis is stated in terms of exact empirical minimizers and identifiable population optima. In the implementation, parameter estimates are obtained numerically with BFGS and therefore need not attain the corresponding global empirical minima. As a diagnostic of obvious optimization failure, we checked the ordering implied by local minimization,
\begin{equation}
    L_k(\bar{\theta}) \geq L_k(\hat{\theta}_k),
\end{equation}
which holds for 98.9\% of the evaluated candidate-partition pairs. Violations are therefore rare, indicating that the observed DSCA averaging effect is not driven by obvious failures of the local optimizer.

Candidate expressions are also symbolically simplified before calibration. This preprocessing removes exact algebraic redundancies whenever they can be identified symbolically, thereby strengthening the correspondence with the identifiability assumption used above and improving the conditioning of the subsequent calibration step. Residual near-singular or weakly identified parameter directions may still arise from the local statistical geometry of the candidate expression.

\section{Experimental Setup} \label{sec:setup}
\subsection{Genetic Programming Framework}

For our experimental evaluations, we employ a standard tree-based Genetic Programming algorithm as the foundational framework. Because the contribution of this work concerns the coefficient-calibration mechanism used during individual evaluation rather than the evolutionary operators themselves, we keep the GP configuration fixed across all compared methods. The complete set of evolutionary hyperparameters is detailed in Table \ref{tab:ea_settings} and explained in the following.

The initial population is generated using a random initialization method with decay. In this approach, tree depth is regulated by two hyperparameters: parsimony (representing the initial operator probability at the root node) and a parsimony decay (which reduces this probability as tree depth increases). During the evolutionary cycle, parent selection is performed via tournament selection. The population is subsequently evolved using a uniformly distributed set of genetic operators to maintain structural diversity.

The critical distinguishing feature of our implementation is the integration of the proposed coefficient calibration, which transforms the standard GP into a memetic algorithm. After variation generates a new individual, its expression is symbolically simplified and its free numerical parameters are calibrated before fitness evaluation. The calibration rule is the only component changed across the compared methods.
The memetic phase proceeds as follows for each newly generated individual:
\begin{enumerate}
    \item \textbf{Parameter Extraction:} After symbolic simplification, the tree is traversed to identify and extract all free constants and parameterizable operator coefficients.
    \item \textbf{Local Search:} The selected calibration procedure is applied to the extracted parameters. Centralized baselines optimize them on the complete optimization set, whereas DSCA performs the partition-wise optimization and averaging procedure described in Section~\ref{sec:methodology}.
    \item \textbf{Parameter Injection:} Following a Lamarckian evolutionary scheme, the resulting calibrated parameters are directly injected back into the individual's symbolic tree, permanently updating its genetic material.
    
\end{enumerate}

Only after this inner calibration step terminates, the next generation of individuals is selected. Specifically, each calibrated individual is scored according to one of the three selection criteria and the survival strategy is applied to select the best-scoring individuals from the combined parent-offspring pool.

Coefficient calibration therefore determines the parameter values at which an individual is evaluated, whereas the outer criterion determines how calibrated individuals compete during evolution. To test whether the effect of DSCA depends on a particular outer selection rule, we evaluate each calibration strategy under three criteria commonly used in SR: the negative log-likelihood (NLL), AIC, and BIC \cite{soltani2026comparative,ramlan2026,kronberger2026guiding}. These criteria span unpenalized likelihood fitting and two standard levels of parameter-count penalization, $2k$ and $k\ln N$, where $k$ is the number of free parameters and $N$ the number of training samples. Recent comparative studies indicate that they form a compact and competitive set, since more elaborate alternatives are not uniformly superior when used directly to guide evolutionary search \cite{soltani2026comparative,ramlan2026,kronberger2026guiding}. The purpose of this comparison is thus not to benchmark model-selection criteria, but to assess whether the effect of DSCA is robust across them.

\begin{table}[htbp]
    \centering
    \caption{Evolutionary and numerical optimization hyperparameters.}
    \label{tab:ea_settings}
    \begin{tabular}{@{} p{0.4\linewidth} p{0.55\linewidth} @{}}
        \toprule
        \textbf{Parameter} & \textbf{Value/Setting} \\
        \midrule
        Population size ($\mu$) & $100$ \\
        Offspring size ($\lambda$) & $100$ \\
        Number of generations & $100$ \\
        Initialization method & Random with decay \\
        Parsimony  & $0.8$ \\
        Parsimony decay & $0.85$ \\
        Maximum tree size & $75$ nodes \\
        Function set & $\{+, -, \times, \div, (\cdot)^2, (\cdot)^3,$ $\sin, \exp, \ln\}$ \\
        Terminal set & Independent variables $X_1, \dots, X_n$, ERCs $\in [-10, 10]$ \\
        Selection method & Tournament selection (size $k=3$) \\
        Genetic Operators & Crossover, Subtree mutation, Insert node, Delete node, Mutate leaf, Mutate operator \\
         Selection criteria & [NLL, AIC, BIC] \\
        Survival strategy & $(\mu + \lambda)$ selection (best $100$ survive) \\
        \midrule
        \multicolumn{2}{c}{\textbf{Numerical Optimizer Settings}} \\
        \midrule
        Optimization frequency & Applied to all newly generated individuals \\
        \multicolumn{2}{@{}p{\linewidth}@{}}{\textit{Proposed Method: DSCA}} \\
        \hspace{4mm} Base optimizer & BFGS (settings as in Baseline 1) \\
        \hspace{4mm} Number of partitions ($K$) & $5$ \\
        \hspace{4mm} Number of strata ($S$) & $10$ \\
        \hspace{4mm} Heterogeneity ($\alpha$) & $1$ \\
        \hspace{4mm} Stratification score & Ridge prediction $g(x)$ \\
        \hspace{4mm} Surrogate fraction ($\rho$) & $15\%$ of the training set \\
        \hspace{4mm} Ridge penalty & Leave-one-out CV over $\{10^{-3},\dots,10^{3}\}$ \\
        \multicolumn{2}{@{}p{\linewidth}@{}}{\textit{Baseline 1: BFGS (via SciPy `minimize')}} \\
        \hspace{4mm} Jacobian estimation & 2-point finite difference \\
        \hspace{4mm} Tolerance (\texttt{tol}) & None (relies on defaults) \\
        \multicolumn{2}{@{}p{\linewidth}@{}}{\textit{Baseline 2: Levenberg-Marquardt (via SciPy `least\_squares')}} \\
        \hspace{4mm} Jacobian estimation & 2-point finite difference \\
        \hspace{4mm} Cost change (\texttt{ftol}) & $1 \times 10^{-8}$ \\
        \hspace{4mm} Step change (\texttt{xtol}) & $1 \times 10^{-8}$ \\
        \hspace{4mm} Gradient norm (\texttt{gtol}) & $1 \times 10^{-8}$ \\
        \bottomrule
    \end{tabular}
\end{table}

\subsection{Baselines}
We compare against two established centralized calibration methods: Levenberg-Marquardt (LM), the standard local search in SR frameworks such as Operon \cite{burlacu2020}, and BFGS, the default optimizer in PySR \cite{cranmer2023}. Both optimize the same within-individual objective as the local fits of DSCA, so differences in evolutionary behavior can be attributed to the calibration strategy rather than to a change in loss function. Throughout, prediction errors are assumed independent and
normally distributed, under which minimizing the NLL with respect to
$\theta$ is equivalent to least-squares fitting. 

To ensure a fair comparison, all baseline optimizers are integrated into the memetic framework in exactly the same manner as our proposed approach. Table \ref{tab:ea_settings} details the hyperparameters used for these algorithms, which follow standard default configurations. In addition, our ablation experiments include a multi-start centralized BFGS control, in which the same structure is optimized repeatedly on the pooled data. This will help distinguish the effect of heterogeneous averaging from the simpler effect of allocating multiple optimization attempts to each candidate.

One asymmetry deserves to be stated explicitly. DSCA reserves $15\%$ of the training samples to fit the stratification surrogate and never uses them again, so it optimizes constants and evaluates fitness on $N' = 0.85N$ samples, whereas the two baselines use all $N$. We kept the nominal training size fixed across methods rather than equalizing $N'$, so the comparison is conservative with respect to our own method.

\subsection{Synthetic Track}

The first track evaluates model recall on synthetic datasets where the true generating equations are known, allowing us to explicitly control the injected noise. Table \ref{tab:synthetic_datasets} shows the generating functions tested. We selected these from widely used benchmarks to ensure a diverse representation of functional forms and operators, including polynomials, rational expressions, exponentials, and trigonometric functions.

We simulate two noise regimes by adopting the methodology from \cite{de2024srbench++}. Specifically, we add Gaussian noise to the target values:
\begin{equation}
    y_{noise} \sim \mathcal{N}\left(y, \sigma_y \sqrt{\frac{ratio}{1 - ratio}}\right)
\end{equation}
where $y$ is the true target value, $\sigma_y$ is the standard deviation of the target, and $ratio$ determines the noise level. For our experiments, we use noise ratios of 0.025 and 0.1, corresponding to the \emph{easy} and \emph{hard} levels in \cite{de2024srbench++}, and both 100  and 200 samples for the training set.

We consider an equation successfully recovered if it achieves an $R^2 > 0.99$ on the noiseless test set. Some studies verify exact symbolic equivalence by checking if the ratio or difference relative to the generating function is constant (typically using libraries like SymPy). We initially tested this methodology but observed too many false negatives, as computer algebra systems frequently fail to simplify structurally equivalent expressions to a common syntactic form.

\begin{table*}[htbp]
\centering
\renewcommand{\arraystretch}{2}
\caption{Generating functions and variable distributions for the synthetic benchmark track.}
\label{tab:synthetic_datasets}
\begin{tabular}{c c c c c}
\hline
Function & Generating Function & Distribution & \textbf{$d$} & Reference \\
\hline
$f_1(x)$ & $8x_1^2 + 8x_2^3 - 15$ & $x_i \sim \mathcal{U}(-3, 3)$ & 2 & Jin-2 \cite{topchy2001, jin2019bayesian, chen2015generalisation} \\
$f_2(x)$ & $2.5x_1^4 - 1.3x_1^3 + 0.5x_2^2 - 1.7x_2$ & $x_i \sim \mathcal{U}(-3, 3)$ & 2 & Jin-1 \cite{topchy2001,jin2019bayesian, chen2015generalisation} \\
$f_3(x)$ & $1.35x_1x_2 + 5.5\sin((x_1-1)(x_2-1))$ & $x_i \sim \mathcal{U}(-3, 3)$ & 2 & Jin-6 \cite{topchy2001,jin2019bayesian, chen2015generalisation} \\
$f_4(x)$ & $\frac{e^{-(x_1-1)^2}}{1.2+(x_2-2.5)^2}$ & $x_i \sim \mathcal{U}(0.3, 4)$ & 2 & Vladislavleva-1 \cite{vladislavleva2008order} \\
$f_5(x)$ & $\frac{30(x_1-1)(x_3-1)}{(x_1-10)x_2^2}$ & $x_1, x_3 \sim \mathcal{U}(0.05, 2), x_2 \sim \mathcal{U}(1, 2)$ & 3 & Vladislavleva-5 \cite{vladislavleva2008order} \\
\hline
\end{tabular}
\end{table*}

\subsection{Real-World Track}

For the second track, we selected 10 real-world datasets from the PMLB collection \cite{olson2017pmlb}, which serves as a standard benchmark within the SR community \cite{la2021contemporary, de2024srbench++}. We chose these specific datasets to ensure a diverse representation of sample sizes, dimensionalities, ratios of binary to continuous variables, and overall difficulty levels for SR algorithms. Table \ref{tab:datasets} details the characteristics of each dataset, including the total number of instances and the test set sizes.

To ensure statistical robustness, we conduct 10 independent runs for each dataset and sample size regime. To prevent data leakage and simulate these different regimes, we employ a splitting strategy for each run. First, we reserve 25\% of the data, sampled uniformly at random, to serve as the test set. From the remaining 75\%, we again sample uniformly at random to extract the necessary number of instances to form the training set. Specifically, we evaluate three training set size regimes: 100, 200, and 500 samples.

\begin{table}[t]
\renewcommand{\arraystretch}{1.35}
\caption{Characteristics of the selected PMLB datasets.} \label{tab:datasets}
\begin{center}
\begin{tabular}{lrrrrr}
\textbf{Dataset}  &\textbf{$d$} &\textbf{$B$}  &\textbf{$\mathbb R$} &\textbf{total} &\textbf{test}\\
\hline
Fri\_c1   &25 &0 &25 &1000 &250\\
Fri\_c2   &5 &0 &5 &1000 &250\\
Pollen    &4 &0 &4 &3790 &948\\
Puma8NH    &8 &0 &8 &8192 &2048\\
Echo M.   &9 &3 &6 &17233 &4309\\
Houses    &8 &0 &8 &20330 &5083\\
BNG lowbwt    &15 &11 &4 &31104 &7776\\
2dplanes    &10 &10 &0 &40768 &10192\\
Breast T.    &13 &9 &4 &114890 &28723\\
BNG pbc    &18 &6 &12 &1000000 &250000\\
\hline
\end{tabular}
\end{center}
\end{table}

\section{Results} \label{sec:results}

\subsection{Synthetic Track}

Table~\ref{tab:wtl_synthetic} summarizes the recovery rates on the synthetic benchmarks. For every combination of sample size, noise level, and selection criterion, it reports on how many of the five generating functions DSCA attains a higher (win), equal (tie), or lower (loss) success rate than each baseline over 10 seeds. DSCA wins in 64 of the 120 function-level comparisons, ties in 48, and loses in 8. It never loses on more than one function in any configuration, and it never loses to BFGS in the 200-sample regime. 

Per-function success rates are reported in the supplementary material. Ties are largely driven by $f_3$, on which all methods fail almost systematically. All losses occur on the rational functions $f_4$ and $f_5$, where the three optimizers perform comparably and differences are typically within one or two seeds. The wins are concentrated on the polynomial functions $f_1$ and $f_2$, and this holds in both noise regimes.

It is worth noting where the advantage is largest. On $f_1$ and $f_2$, which have well-conditioned parameterizations, DSCA gains most under the unregularized NLL criterion, where the baselines have no defense against overparameterized alternatives. Under BIC the gap narrows, since the two mechanisms then act on partly overlapping failure modes. This pattern recurs, more clearly, on the real-world track.

Table \ref{tab:synthetic_success_count_combined} presents the aggregated success counts across the different synthetic datasets. These results demonstrate the superior performance of our proposed approach compared to the baselines, maintaining this advantage both overall and within each individual model selection criterion.

\begin{table}[htpb]
    \centering
    \caption{Win/tie/loss counts of DSCA against each baseline over the synthetic datasets, by sample size, noise level, and selection criterion.}
    \label{tab:wtl_synthetic}
    \begin{tabular}{llcccccc}
        \toprule
        & & \multicolumn{2}{c}{\textbf{NLL}} & \multicolumn{2}{c}{\textbf{AIC}} & \multicolumn{2}{c}{\textbf{BIC}} \\
        \cmidrule(lr){3-4} \cmidrule(lr){5-6} \cmidrule(lr){7-8}
        \textbf{$\boldsymbol{N}$} & \textbf{Noise} & \textbf{BFGS} & \textbf{LM} & \textbf{BFGS} & \textbf{LM} & \textbf{BFGS} & \textbf{LM} \\
        \midrule
        \multirow{2}{*}{\textbf{100}} & Low  & 3/1/1 & 4/1/0 & 3/1/1 & 3/1/1 & 3/2/0 & 1/3/1 \\
                                      & High & 2/3/0 & 2/3/0 & 2/3/0 & 2/3/0 & 4/1/0 & 2/2/1 \\
        \midrule
        \multirow{2}{*}{\textbf{200}} & Low  & 3/2/0 & 3/2/0 & 2/3/0 & 3/2/0 & 2/3/0 & 2/2/1 \\
                                      & High & 3/2/0 & 3/1/1 & 3/2/0 & 3/2/0 & 3/2/0 & 3/1/1 \\
        \bottomrule
    \end{tabular}
\end{table}

\begin{table}[htpb]
    \centering
    \caption{Aggregate success counts on synthetic datasets by sample size and selection criterion.}
    \label{tab:synthetic_success_count_combined}
    \begin{tabular}{l ccc ccc}
        \toprule
        & \multicolumn{3}{c}{\textbf{100 Samples}} & \multicolumn{3}{c}{\textbf{200 Samples}} \\
        \cmidrule(lr){2-4} \cmidrule(lr){5-7}
        \textbf{Algorithm} & \textbf{NLL} & \textbf{AIC} & \textbf{BIC} & \textbf{NLL} & \textbf{AIC} & \textbf{BIC} \\
        \midrule
        \textbf{DSCA (our)} & \textbf{26} & \textbf{31} & \textbf{43} & \textbf{43} & \textbf{49} & \textbf{55} \\
        BFGS & 14 & 14 & 30 & 28 & 31 & 48 \\
        LM   & 5  & 13 & 37 & 24 & 30 & 49 \\
        \bottomrule
    \end{tabular}
\end{table}

\subsection{Real-World Track}\label{sec:realworld}

\subsubsection{Bayesian Statistical Analysis}
Before comparing the algorithms, we aggregate each configuration by taking the median over the 10 independent seeds of every dataset, and we clip $R^2$ values below $-1$ so that degenerate fits do not dominate the comparison. The same aggregation is used for the analysis reported in the remainder of this section.

We compare DSCA against BFGS and LM using the Bayesian Wilcoxon signed-rank test \cite{benavoli2014bayesian,benavoli2017time}. The test provides posterior probabilities that DSCA is better, practically equivalent, or worse than each baseline. We perform the comparison separately for predictive accuracy, measured by $R^2$, and model complexity, measured by the number of expression-tree nodes. Practical equivalence is defined through Regions of Practical Equivalence (ROPEs): $[-0.01,0.01]$ for differences in $R^2$ and $[-2,2]$ nodes for differences in complexity. Differences within these intervals are treated as practically negligible.

Figure~\ref{fig:bayes_test} reports the posterior comparison for predictive accuracy and model complexity. In the 100- and 200-sample regimes, DSCA generally has a high posterior probability of outperforming both centralized baselines. As the sample size
increases, this advantage narrows. This convergence in predictive accuracy should not be interpreted as a consequence of a vanishing population DSCA effect: Eq.~\eqref{eq:dichotomy} shows that, for misspecified candidates with $\Delta>0$, increasing $N'$ only reduces the sampling fluctuations around the persistent excess $\Delta$. Rather, with more data the search has less room to retain structures whose apparent advantage depends mainly on finite-sample calibration effects. The convergence is particularly visible under BIC, whose strong complexity pressure and M-closed interpretation can substantially restrict the set of competitive structures. At $N=500$ under BIC, DSCA also closely tracks its centralized BFGS counterpart in predictive accuracy, suggesting that residual differences may partly reflect the shared base optimizer rather than the averaging mechanism. Importantly, in Figure~\ref{fig:bayes_test} the same convergence is not observed for complexity: DSCA retains a high posterior probability of producing smaller expressions across configurations. This divergence suggests a non-trivial interplay between heterogeneous calibration and explicit complexity control: stronger outer penalization reduces the predictive differences between methods without eliminating the structural effect induced by DSCA.

\begin{figure}
    \centering
    \includegraphics[width=0.99\linewidth]{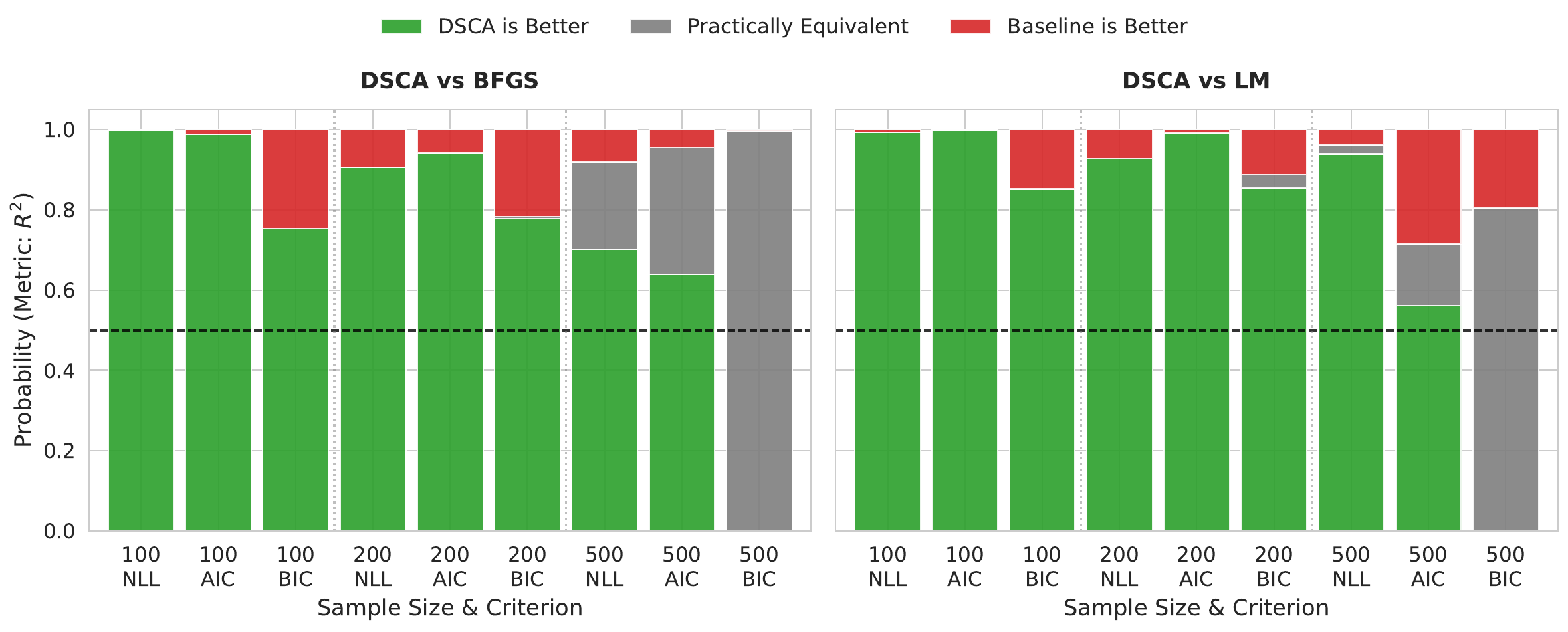}
    \includegraphics[width=0.99\linewidth]{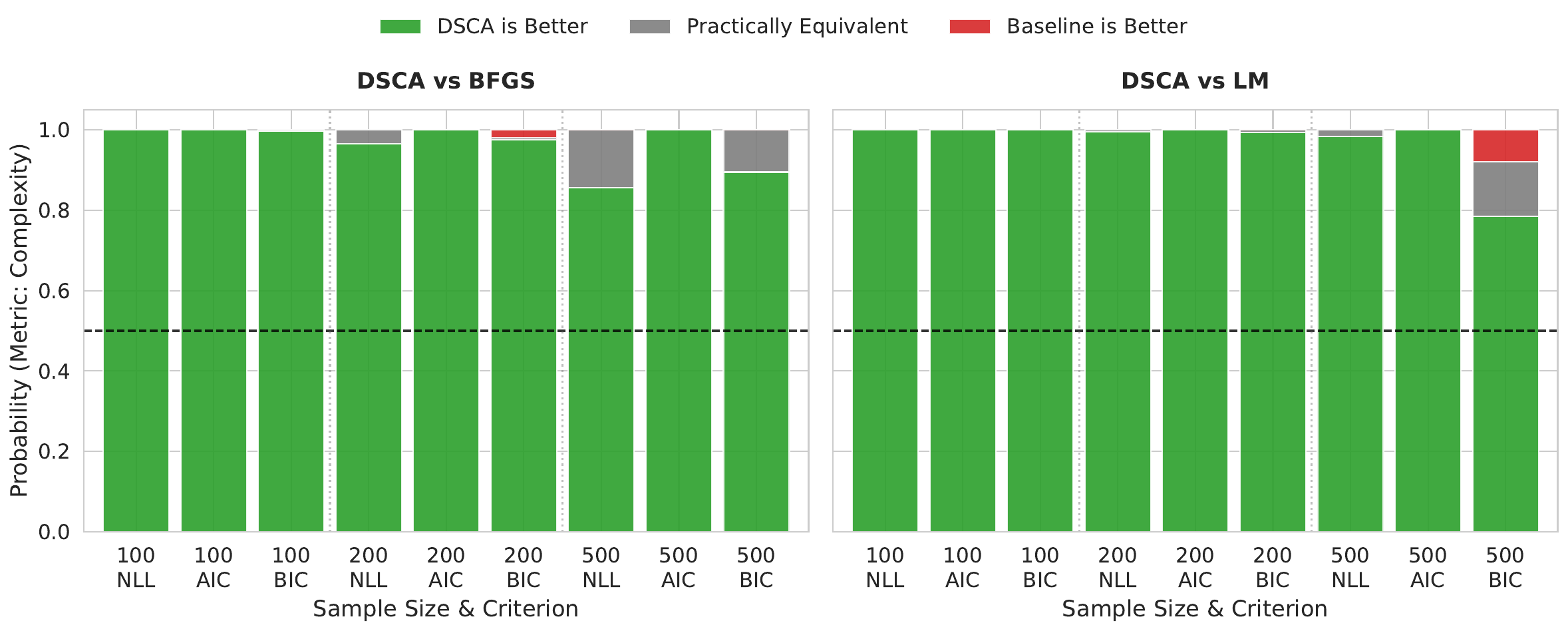}
    \caption{Bayesian Wilcoxon signed-rank test results evaluating predictive accuracy (top panels) and model complexity (bottom panels) against the baselines.}
    \label{fig:bayes_test}
\end{figure}

\subsubsection{Absolute Performance Differences}

While rank-based evaluations effectively illustrate relative positioning, they inherently hide the absolute magnitudes of performance gaps between algorithms. To provide a concrete measure of these performance variations, Tables \ref{tab:diff_table_r2} and \ref{tab:diff_table_comp} detail the absolute differences between the evaluated algorithms and a standard baseline. We selected LM as the baseline, given its widespread adoption within the SR community.

To ensure results are comparable across datasets that yield different baseline accuracies and equation sizes, we aggregate the scores by first calculating the median performance across seeds for both the evaluated algorithm and the baseline, and then computing the difference. For a given algorithm $a$, evaluated on dataset $d$, we define the aggregated difference $\Delta_a$ as:
\begin{equation}
\begin{aligned}\Delta_a = \frac{1}{N_D}\sum_{d=1}^{N_D}&\Big(\text{med}_{i=1}^{10}(\text{Score}_{a, d, i}) - \text{med}_{i=1}^{10}(\text{Score}_{\text{LM}, d, i})\Big)\end{aligned}    
\end{equation}
where $N_D$ is the total number of datasets.

\begin{table}[]
\renewcommand{\arraystretch}{1.25}
\caption{Predictive accuracy ($R^2$) differences ($\Delta_a$) relative to the Levenberg-Marquardt baseline (Mean $\pm$ Std).}
    \centering
    \begin{tabular}{llcc}
\toprule
 & Algorithm & BFGS & DSCA (our) \\
Criterion & $N$ samples &  &  \\
\midrule
\multirow[t]{3}{*}{NLL} & 100 & 0.26 $\pm$ 0.57 & \textbf{0.85 $\pm$ 0.74} \\
 & 200 & 0.28 $\pm$ 0.60 & \textbf{0.40 $\pm$ 0.63} \\
 & 500 & 0.02 $\pm$ 0.04 & \textbf{0.14 $\pm$ 0.40} \\
 \cline{1-4}
\multirow[t]{3}{*}{AIC} & 100 & 0.30 $\pm$ 0.43 & \textbf{0.78 $\pm$ 0.66} \\
 & 200 & 0.36 $\pm$ 0.69 & \textbf{0.54 $\pm$ 0.67} \\
 & 500 & 0.00 $\pm$ 0.03 & \textbf{0.01 $\pm$ 0.05} \\
\cline{1-4}
\multirow[t]{3}{*}{BIC} & 100 & 0.03 $\pm$ 0.20& \textbf{0.13 $\pm$ 0.29} \\
 & 200 & 0.00 $\pm$ 0.04& \textbf{0.01 $\pm$ 0.05} \\
 & 500 & -0.01 $\pm$ 0.02 & -0.01 $\pm$ 0.02 \\

\cline{1-4}
\bottomrule
\end{tabular}

    \label{tab:diff_table_r2}
\end{table}

\begin{table}[]
\renewcommand{\arraystretch}{1.25}
\caption{Model complexity differences ($\Delta_a$) relative to the Levenberg-Marquardt baseline (Mean $\pm$ Std).}
    \centering
    \begin{tabular}{llcc}
\toprule
 & Algorithm & BFGS & DSCA (our) \\
Criterion & $N$ samples &  &  \\
\midrule
\multirow[t]{3}{*}{NLL} & 100 & -0.2 $\pm$ 1.1 & \textbf{-15.8 $\pm$ 10.8} \\
 & 200 & -0.6 $\pm$ 2.3 & \textbf{-7.5 $\pm$ 6.2} \\
 & 500 & -1.2 $\pm$ 2.9 & \textbf{-4.0 $\pm$ 2.8} \\
\cline{1-4}
\multirow[t]{3}{*}{AIC} & 100 & -0.8 $\pm$ 2.8 & \textbf{-23.3 $\pm$ 11.4} \\
 & 200 & -2.1 $\pm$ 4.9 & \textbf{-15.1 $\pm$ 13.3} \\
 & 500 & 0.0 $\pm$ 2.5 & \textbf{-7.6 $\pm$ 5.9} \\
\cline{1-4}
\multirow[t]{3}{*}{BIC} & 100 & -1.9 $\pm$ 4.0 & \textbf{-9.3 $\pm$ 7.1} \\
 & 200 & -1.4 $\pm$ 5.9 & \textbf{-7.4 $\pm$ 8.3} \\
 & 500 & 0.1 $\pm$ 5.5 & \textbf{-3.0 $\pm$ 7.1} \\
\cline{1-4}

\bottomrule
\end{tabular}

    \label{tab:diff_table_comp}
\end{table}

BFGS returns solutions that are remarkably similar to LM: its accuracy gains are limited to the unregularized and weakly penalized criteria in the smaller samples, and the expressions it retains are essentially of the same size as the baseline throughout. DSCA instead improves on LM along both axes at once. It is the more accurate of the two methods in every configuration in which a difference is present, with gaps that are practically relevant under NLL and  AIC at 100 and 200 samples, and it does so while removing a substantial number of nodes from the final equations, up to roughly 24 in the lowest sample regime. The size reduction is the more systematic of the two effects, as it persists across all three criteria and at every sample size, whereas the accuracy advantage concentrates where the parameter estimates are least constrained by the data. As in the previous analyses, the differences narrow at 500 samples and under BIC, where all methods essentially agree in accuracy and the remaining
advantage of DSCA is one of parsimony.

\subsubsection{Computational Efficiency}
Training time is an important secondary metric for SR. Since DSCA performs $K$ independent inner optimizations per candidate, it is expected to be more computationally intensive. Figure \ref{fig:time_plot} reports the training time of all evaluated algorithms\footnote{All experiments were run on a SLURM-managed cluster whose nodes have two Intel Xeon Gold 5320 processors (2.2 GHz, 52 cores in total) and 512 GB of shared memory. To ensure a fair comparison, every job was allocated four CPU cores and 16 GiB of RAM.}.

Despite the additional local optimizations, DSCA remains competitive with, or faster than, the centralized baselines in the lower-sample regimes. With 100 samples, it runs up to 20\% faster than LM and BFGS under NLL and AIC. With 200 samples, the differences are mixed, ranging from about 9\% faster than BFGS under AIC to 12--20\% slower depending on the criterion. Only with 500 samples does the overhead of the inner loops fully materialize, with DSCA running 15--26\% slower than the baselines across all criteria.

This non-linear scaling behavior occurs because numerical optimization is only one component of the SR pipeline, whose runtime is dominated by other operations such as symbolic simplification and offspring generation. As these operations are faster on smaller expressions, the simpler populations maintained by DSCA accelerate the overall evolutionary process and, in low-sample regimes, fully offset the cost of the inner optimization loops.

\begin{figure*}
    \centering
    \includegraphics[width=0.8\linewidth]{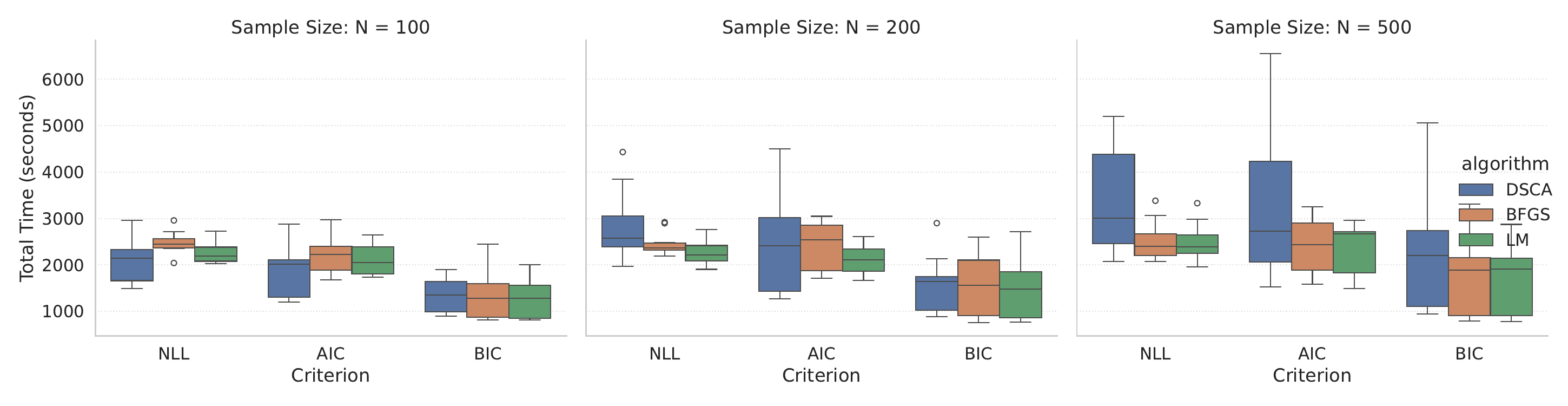}
    \caption{Algorithmic training time (in seconds) across different selection criteria and sample size regimes.}
    \label{fig:time_plot}
\end{figure*}

\subsection{Mechanism Analysis: Selection-Relevant Excess and Generalization}
\label{exp:mechanism}

The preceding experiments show that DSCA tends to favor smaller expressions while retaining competitive predictive accuracy, particularly in low-sample regimes. We now examine whether this behavior is consistent with the mechanism characterized in Section~\ref{sec:theory}, namely the selection-relevant excess loss $\Delta_{N'}$ defined in \eqref{eq:delta}. The objective is not to establish a causal identification of the mechanism, but to test two empirical implications of the theory: whether the structures favored by DSCA exhibit smaller out-of-sample deterioration under a common centralized calibration, and whether $\Delta_{N'}$ is associated with that deterioration across structures.

We pool the final populations produced by DSCA, LM, and BFGS across all real-world datasets, all three selection criteria, and all ten seeds in the $N=200$ regime, and remove duplicate structures. Each surviving structure is then calibrated twice on the same optimization data: once centrally with BFGS, yielding $\hat\theta_{N'}$ and the centralized training loss $L_{N'}(\hat\theta_{N'})$, and once with DSCA, yielding $\bar\theta$ and $L_{N'}(\bar\theta)$. Their difference is $\Delta_{N'}$. Since each structure is scored by the same two procedures regardless of its origin, this design separates the properties of the structure from the properties of the optimizer that produced it, and lets us ask whether $\Delta_{N'}$ is informative about structures that DSCA itself never selected.

Because the pooled datasets differ in target scale, all MSE quantities are divided by the median training MSE of all models on the same dataset, so that $\Delta_{N'}$ and the generalization gap are expressed relative to the typical error level of each dataset.

\begin{figure*}
    \centering
    \includegraphics[width=0.7\linewidth]{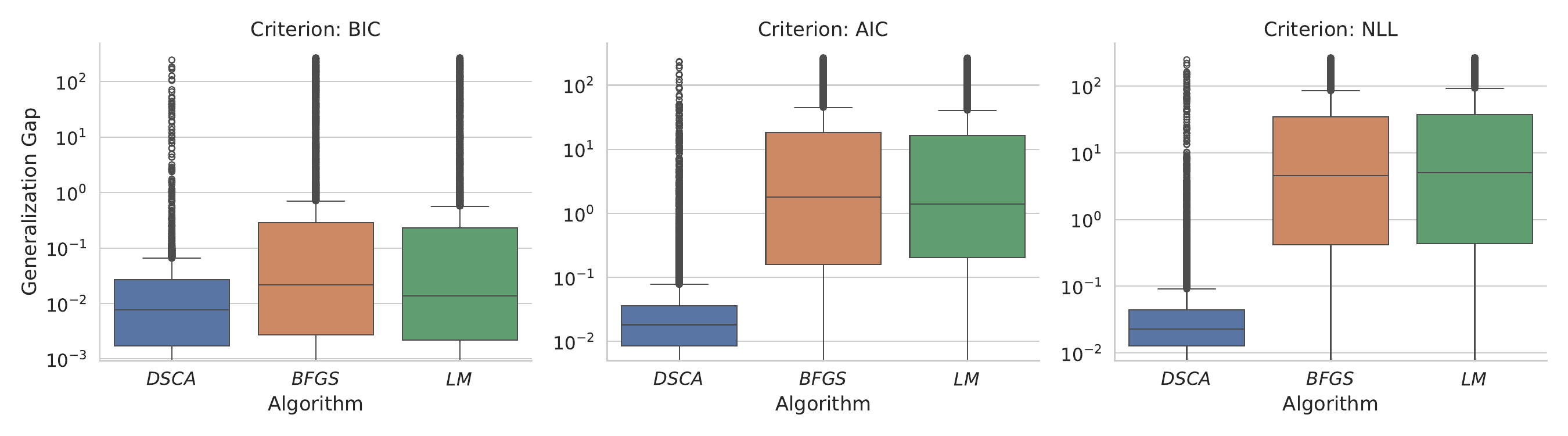}
    \caption{Generalization gap of the pooled final populations, by the optimizer that produced
each structure (200-sample regime, centralized re-fit).}
    \label{fig:mse_train_lm}
\end{figure*}

\subsubsection{Structures found under centralized optimization generalize worse} Figure~\ref{fig:mse_train_lm} reports the generalization gap of the pooled structures after all of them have been recalibrated centrally. This common refit is important: differences in the gap then reflect the structures retained by the different evolutionary procedures rather than the particular coefficient values produced during their original runs. Across the three outer selection criteria, structures originating from centralized LM and BFGS runs exhibit substantially larger median generalization gaps and heavier upper tails than structures originating from DSCA. Combined with the predictive-accuracy results of Section~\ref{sec:results}-B, this suggests that DSCA tends to retain structures whose fitted performance is less specific to the training sample, rather than merely producing structures with uniformly poorer in-sample fit.

\subsubsection{The excess loss is a signal of overfitting} Figure~\ref{fig:mse_test_dsca} relates $\Delta_{N'}$ to the generalization gap over the same pooled set. The association is strong and stable across criteria, with Pearson correlations between $0.72$ and $0.83$, and it holds within each generating algorithm, not only across them: structures produced by LM with a large $\Delta_{N'}$ generalize as poorly as structures produced by any other optimizer with a comparable $\Delta_{N'}$. This is the property the mechanism requires. $\Delta_{N'}$ is computed without access to test data, from the training sample alone, and is available at no additional cost during the search, since DSCA uses $L_{N'}(\bar\theta)$ as its fitness in any case.

\begin{figure*}
    \centering
    \includegraphics[width=0.9\linewidth]{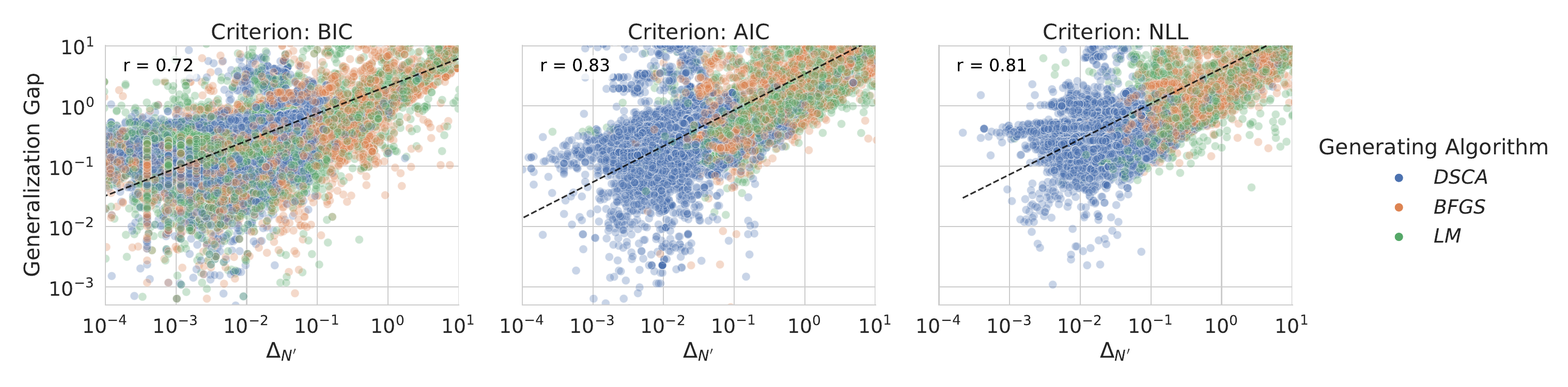}
    \caption{Excess loss $\Delta_{N'}$ against the generalization gap for the pooled final
populations, colored by the optimizer that produced the structure.}
    
    \label{fig:mse_test_dsca}
\end{figure*}

\subsection{Mechanism Controls: Optimization Effort and Partition Heterogeneity}
\label{exp:ablation}

The mechanism analyzed in Section~\ref{sec:theory} attributes the behavior of
DSCA to the averaging of local optima obtained on partitions drawn from environments that differ in their covariate distribution. 
Two competing explanations must be excluded. First, the advantage might arise simply because DSCA performs multiple optimization attempts per candidate rather than a single centralized fit. Second, partitioning itself might be sufficient, irrespective of whether the subsets exhibit the controlled covariate heterogeneity targeted by DSCA.

We therefore compare four variants under the 200-sample regime. \emph{Multistart} performs $K$ independent runs of the base optimizer on the full training set from random starting points and retains the best solution, so it matches the number of optimization attempts performed by DSCA without partitioning the data. $R_{\text{batch}}$ partitions the data uniformly at random. The remaining two variants use the proposed Dirichlet-Sinkhorn scheme with $\alpha=100$, which yields nearly homogeneous partitions, and with $\alpha=1$, which is the configuration used throughout the paper. Each partition-based configuration is evaluated with $K=3$ and $K=5$.

To prevent datasets of differing intrinsic difficulty from dominating the averages, we report a regret gap rather than raw scores. For a given dataset and selection criterion, the gap measures how far a configuration's $R^2$ falls below the best $R^2$ attained by any configuration, and how far its complexity rises above the smallest complexity attained by any configuration. A gap of zero denotes the best observed behavior and lower is better throughout.

Figures~\ref{fig:grid_200_r2} and~\ref{fig:grid_200_comp} report the two gaps
averaged over the real-world datasets (200 samples). On predictive accuracy the multi-start
variant is the strongest or tied-strongest configuration under every criterion,
consistent with the advantage of selecting the best solution among several
independent optimizations of the same pooled objective; unlike DSCA, Multistart
does not intentionally modify the parameter vector at which the candidate is
evaluated. The partition-based variants give up a small amount of $R^2$ relative
to it, most visibly at $K=3$ under the unregularized criteria, and raising $K$
to $5$ removes most of this deficit.

The ordering reverses on model complexity, and this is where the specific
contribution of the proposed scheme becomes visible. Multi-start generally returns by far
the largest expressions, and increasing the number of restarts does not reduce
their size. This is the expected consequence of a procedure that improves the fit
of every candidate uniformly, including the overparameterized ones, and it
confirms that the parsimony observed in Section~\ref{sec:realworld} cannot be
attributed to the additional optimization effort alone. Both partitioning
strategies produce markedly smaller models, and among them the heterogeneous
stratification is consistently the most parsimonious at $K=5$, ahead of both
random partitioning and the homogeneous split. The three strategies are much
closer at $K=3$ under NLL, which suggests that few partitions provide too little
contrast between local optima for the averaging penalty to discriminate.

The effect of $K$ is asymmetric across the two axes. Moving from three to five
partitions leaves the accuracy gap essentially unchanged, or slightly improves it, while substantially reducing the complexity gap for the partition-based variants; the same pattern is not observed for Multistart.  Taken together, the two figures support the reading advanced in Section~\ref{sec:theory}: a larger number of narrower environments increases the dispersion of the local optima and hence the loss incurred by averaging, and this penalty falls selectively on candidates whose parameters are unstable across environments. The configuration used in the main experiments, $\alpha=1$ with $K=5$, gives up a small amount of $R^2$ relative to a pure restart strategy in exchange for a reduction of roughly 2 to 15 nodes.

\begin{figure}
    \centering
    \includegraphics[width=\linewidth]{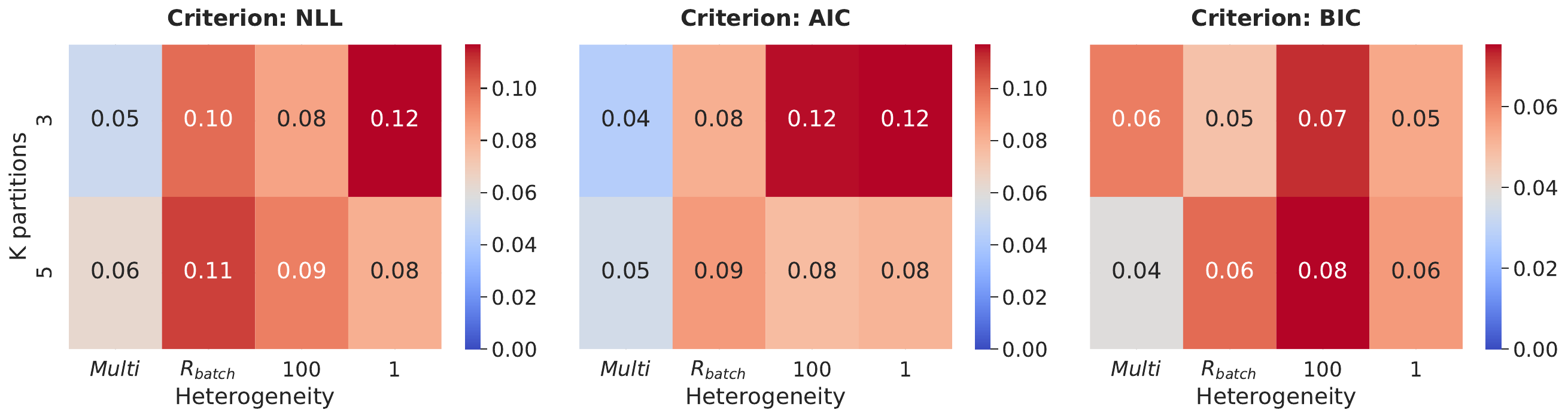}
    \caption{ Ablation study of the predictive accuracy ($R^2$) regret gap under different partitioning strategies.}
    \label{fig:grid_200_r2}
\end{figure}

\begin{figure}
    \centering
    \includegraphics[width=\linewidth]{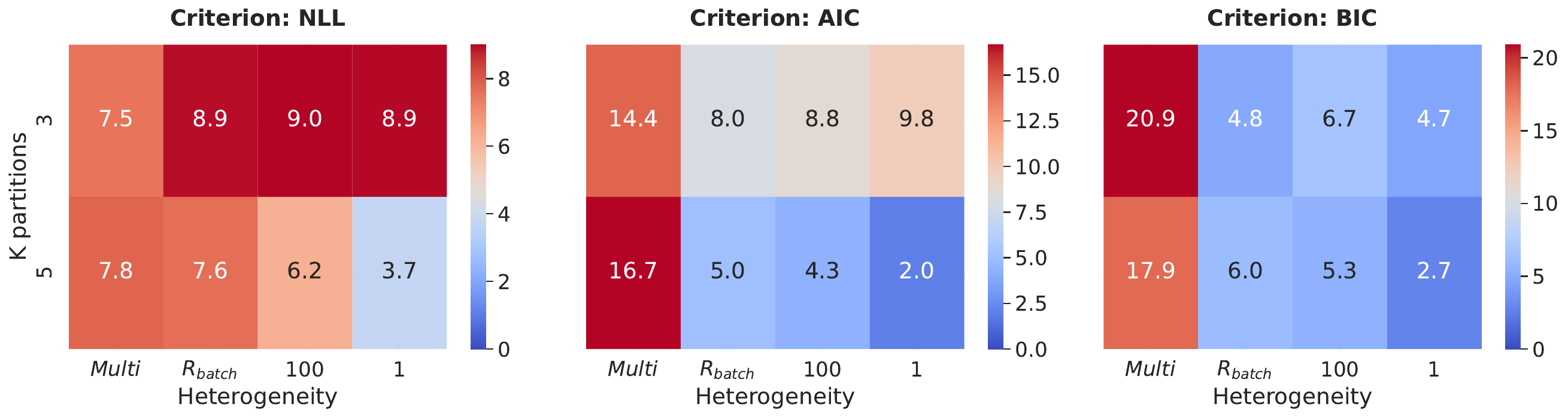}
    \caption{Ablation study of the model complexity regret gap under different partitioning strategies.}

    \label{fig:grid_200_comp}
\end{figure}

\section{Conclusions}\label{sec:conclusions}

This paper introduced Dirichlet-Sinkhorn Constant Averaging, a coefficient-calibration strategy for genetic programming-based SR that repurposes distributed parameter estimation as a source of selection pressure during evolutionary search. DSCA calibrates each candidate independently on subsets whose covariate distributions differ by a controlled amount, and evaluates it using the average of the resulting parameter vectors. Rather than improving the standard local search step, DSCA deliberately changes the parameterization at which candidate structures are compared during evolution.
The analysis identifies the excess loss $\Delta_{N'}$ relative to centralized calibration as the quantity through which DSCA changes selection. Its population counterpart vanishes for correctly specified and identifiable expressions, whereas misspecified candidates can incur a persistent penalty when their environment-specific optima do not aggregate to the pooled optimum.

The experiments are consistent with this interpretation. DSCA achieves higher functional-recovery rates than centralized BFGS and Levenberg-Marquardt on the synthetic benchmarks and a more favorable accuracy-complexity trade-off on the real-world datasets. Its advantage is clearest in low-sample regimes, where sample-specific calibration and overfitting are most likely, and narrows as the sample size grows. After a common centralized re-fit, structures retained by DSCA exhibit smaller generalization gaps, and $\Delta_{N'}$ is positively associated with this gap across structures and algorithms of origin. Finally, multiple centralized restarts do not reproduce DSCA's reduction in expression size, while score-based heterogeneous partitioning yields the most parsimonious models at $K=5$. Together, these results support the role of controlled covariate heterogeneity rather than additional optimization effort alone.

Two directions appear particularly important for future work. First, symbolic simplification removes exact algebraic redundancies but does not prevent weakly identified or near-singular parameterizations. Structural repair operators that eliminate such degeneracies before calibration could make parameter averaging more reliable. Second, the interaction between heterogeneous calibration and outer selection criteria such as AIC and BIC is explored only empirically here. Characterizing how they jointly shape evolutionary pressure remains an important direction for further study.



\bibliographystyle{IEEEtran}
\bibliography{IEEEabrv,references}

\appendices

\section{Theoretical Analysis of Averaging Under Heterogeneity} 
\label{app:theory}

Throughout this appendix, $N$ denotes the size of the optimization set $\mathcal{D}_{\mathrm{opt}}$, denoted $N'$ in the main text, and $n=N/K$. The distinctive effect of DSCA arises from the interaction between heterogeneous partitions and parameter averaging. For a fixed candidate expression, DSCA optimizes its constants independently on $K$ partitions, averages the resulting parameter estimates, and evaluates the averaged parameter vector on the complete training set. This procedure can increase the loss assigned to candidates whose preferred parameters vary across partitions. Two quantities are useful for understanding this effect. The first measures disagreement among locally optimal solutions and explains the mechanism induced by averaging. The second measures the actual increase in loss relative to centralized constant optimization and is therefore the quantity directly relevant to evolutionary selection. 

\medskip \emph{Empirical decomposition.} 
For a candidate expression $f(\cdot,\theta)$, let \begin{equation} 
L_k(\theta):=\frac{1}{n}\sum_{i\in\mathcal D_k}\ell(z_i,\theta), \qquad L_N(\theta):=\frac{1}{N}\sum_{i=1}^N\ell(z_i,\theta), 
\end{equation} 
where $z_i=(x_i,y_i)$ and $n=N/K$. Because the $K$ partitions are disjoint and equally sized, 
\begin{equation}
\label{eq:loss-mixture} 
L_N(\theta)=\frac{1}{K}\sum_{k=1}^K L_k(\theta) 
\end{equation} 
holds exactly for every $\theta$. 
The equal-size constraint is enforced by the Sinkhorn scaling; with unequal partitions the same identity holds with weights $n_k/N$. 
Let 
\begin{equation} 
\begin{array}{cc} 
\hat\theta_k:=\arg\min_\theta L_k(\theta), \qquad \bar\theta:=\frac{1}{K}\sum_{k=1}^K\hat\theta_k,\\[15pt] 
\hat\theta_N:=\arg\min_\theta L_N(\theta). 
\end{array} 
\end{equation} 
Evaluating \eqref{eq:loss-mixture} at $\bar\theta$ gives \begin{equation}
\label{eq:averaging-gap} 
\begin{array}{cc} 
L_N(\bar\theta) =\frac{1}{K}\sum_{k=1}^K L_k(\hat\theta_k)+D_N,\\[15pt] 
D_N:=\frac{1}{K}\sum_{k=1}^K \bigl[L_k(\bar\theta)-L_k(\hat\theta_k)\bigr]. 
\end{array} 
\end{equation} 
If each $\hat\theta_k$ is a global minimizer of its local objective, then $D_N\geq0$. The quantity that measures the loss added by DSCA relative to centralized parameter optimization is 
\begin{equation}
\label{eq:empirical-delta} \Delta_N:=L_N(\bar\theta)-L_N(\hat\theta_N)\geq0. 
\end{equation}

$\Delta_N$ is an analytical quantity rather than an additional component of the algorithm: DSCA assigns the fitness $L_N(\bar\theta)$, whereas $\Delta_N$ measures its excess over the loss obtained by centrally calibrating the same structure. Define also 
\begin{equation} 
F_N:=L_N(\hat\theta_N)-\frac{1}{K}\sum_{k=1}^K L_k(\hat\theta_k). 
\end{equation} 
Then 
\begin{equation}
\label{eq:empirical-decomposition} \boxed{\;D_N=\Delta_N+F_N\;} 
\end{equation} 
Under exact minimization, $D_N$, $\Delta_N$, and $F_N$ are non-negative. The decomposition distinguishes disagreement among local empirical optima from the component that actually changes the loss relative to centralized calibration. In particular, $D_N>0$ does not imply $\Delta_N>0$.

\medskip \emph{Covariate-induced environments.} Let $P$ denote the
data-generating law and let $u=g(x)$ be the stratification score of
Section~\ref{sec:methodology}, with $g$ a function of the covariates alone, so
that $u$ is measurable with respect to $\sigma(x)$. At the population level, let
$\pi_k(u)$ denote the probability of assigning an observation with score $u$ to
partition $k$. The balancing constraints imply
\begin{equation}\label{eq:marginals}
\sum_{k=1}^K\pi_k(u)=1, \qquad \mathbb E_P[\pi_k(g(X))]=\frac{1}{K}.
\end{equation}
Environment $k$ is the law $P_k$ from which $\mathcal D_k$ is drawn, with density
\begin{equation}\label{eq:population-tilt}
p_k(x,y)=K\,\pi_k(g(x))\,p(x,y).
\end{equation}
Consequently, 
\begin{equation}
\label{eq:population-mixture} 
\frac{1}{K}\sum_{k=1}^K p_k(x,y)=p(x,y), \qquad p_k(y\mid x)=p(y\mid x). 
\end{equation} 
Thus the induced environments differ in their covariate marginals while preserving the conditional data-generating mechanism.

\medskip \emph{Population quantities.} 
Define 
\begin{equation} 
\begin{array}{cc}
\Lambda_k(\theta):=\mathbb E_{P_k}[\ell(Z,\theta)],\\[15pt]
\Lambda(\theta):=\mathbb E_P[\ell(Z,\theta)] =\frac{1}{K}\sum_{k=1}^K\Lambda_k(\theta), 
\end{array}
\end{equation} 
the last equality following from \eqref{eq:population-mixture}. 
Let \begin{equation}
\begin{array}{ccc}
\theta_k^*:=\arg\min_\theta\Lambda_k(\theta), \\[10pt]
\bar\theta^*:=\frac{1}{K}\sum_{k=1}^K\theta_k^*,\\[10pt]
\theta^*:=\arg\min_\theta\Lambda(\theta). 
\end{array}
\end{equation} 
Define the selection-relevant population excess \begin{equation}
\label{eq:population-delta} 
\Delta:=\Lambda(\bar\theta^*)-\Lambda(\theta^*)\geq0. 
\end{equation} 
This is the population counterpart of $\Delta_N$ and is the quantity whose behavior is characterized below. 

\medskip 
\begin{proposition}[Correct specification under covariate reweighting]
\label{prop:correct_app}
Suppose that the partitioning score is a function of the covariates only,
$u=g(x)$, and that there exists an identifiable $\theta^*$ satisfying
\[
    f(X,\theta^*)=\mathbb{E}[Y\mid X]
\]
$P$-almost surely. If $\pi_k(g(x))>0$ for $P$-almost every $x$ and every $k$,
then $\theta_k^*=\theta^*$ for every $k$, and consequently $\Delta=0$.
\end{proposition}

\textbf{Proof.}
Because the partitioning score depends only on $x$, \eqref{eq:population-tilt} factorizes as 
\begin{equation} 
\begin{array}{ll}
p_k(x,y) = K\,\pi_k(g(x))\,p(x)\,p(y\mid x) = p_k(x)\,p(y\mid x),\\[10pt]
p_k(x):=K\,\pi_k(g(x))\,p(x), 
\end{array}
\end{equation} so $p_k(y\mid x)=p(y\mid x)$ for every environment. 
The local population risk can therefore be written as \begin{equation}
\label{eq:kl-decomposition} 
\Lambda_k(\theta) = \mathbb E_{P_k} \!\left[\mathrm H\bigl(p(\cdot\mid X)\bigr)\right] + \mathbb E_{P_k} \!\left[ \mathrm{KL}\bigl( p(\cdot\mid X)\,\|\,p_\theta(\cdot\mid X) \bigr) \right]. 
\end{equation} 
The first term is independent of $\theta$, so all $\theta$-dependence of every local objective is carried by 
\[ g_\theta(x) := \mathrm{KL}\bigl( p(\cdot\mid x)\,\|\,p_\theta(\cdot\mid x) \bigr) \geq0. \] 
Under correct specification, $g_{\theta^*}(x)=0$ for $P$-almost every $x$. Since $p_k$ is obtained from $p$ through the non-negative bounded weight $K\pi_k(g(x))\in[0,K]$, every $P$-null set is also $P_k$-null and 
\begin{equation} 
\mathbb E_{P_k}[g_{\theta^*}(X)] = 0 \leq \mathbb E_{P_k}[g_\theta(X)] \qquad \forall\theta,\ \forall k. 
\end{equation} 
Hence $\theta^*$ minimizes every local population risk. Identifiability gives $\theta_k^*=\theta^*$ for every $k$, and therefore $\bar\theta^*=\theta^*$ and $\Delta=0$.\\ \hfill$\square$

\medskip \emph{Relation with squared-error fitting.}  Under squared-error loss,
\begin{equation} 
\begin{array}{lll}
\Lambda_k(\theta) = &\underbrace{ \mathbb E_{P_k}\bigl[\mathrm{Var}(Y\mid X)\bigr] }_{\text{independent of }\theta} \\[20pt]
&+ \mathbb E_{P_k} \!\left[ \bigl(\mathbb E[Y\mid X]-f(X,\theta)\bigr)^2 \right]. 
\end{array}
\end{equation} 
The first term is independent of $\theta$. Proposition~\ref{prop:correct_app} therefore extends directly to squared-error calibration whenever the conditional mean is correctly specified; correct specification of the complete conditional noise distribution is not required. 
\medskip 

\emph{Misspecification and selection-relevant instability.} For a misspecified candidate, the local pseudo-true parameter is \begin{equation} \theta_k^* = \arg\min_\theta \mathbb E_{P_k} \!\left[ \mathrm{KL}\bigl( p(\cdot\mid X)\,\|\,p_\theta(\cdot\mid X) \bigr) \right], \end{equation} a weighted KL projection. If improving the approximation in one region of covariate space necessarily worsens it in another, the optimal resolution of that trade-off depends on the weights assigned to those regions, and the local pseudo-true parameters may therefore differ. Misspecification alone is insufficient for a population-level DSCA effect. A misspecified structure may share the same pseudo-true optimum across the induced environments, and even when the $\theta_k^*$ differ, their arithmetic mean may still coincide with the pooled pseudo-true parameter. The selection-relevant condition is therefore 
\begin{equation}
\label{eq:selection-condition} 
\bar\theta^*\neq\theta^*. 
\end{equation} 
\medskip 

\emph{Local quadratic interpretation.} Suppose that, near the relevant optima, 
\begin{equation} 
\Lambda_k(\theta) = c_k + \frac{1}{2} (\theta-\theta_k^*)^\top H_k (\theta-\theta_k^*), \qquad H_k\succ0. 
\end{equation} 
Let 
\[ \bar H:=\frac{1}{K}\sum_{k=1}^K H_k, \qquad d_k:=\theta_k^*-\bar\theta^*. \] 
Since $\sum_k d_k=0$, the pooled first-order condition gives 
\begin{equation}
\label{eq:curvature-optimum} 
\bar\theta^*-\theta^* = -\bar H^{-1} \frac{1}{K} \sum_{k=1}^K (H_k-\bar H)d_k. 
\end{equation} 
Consequently, 
\begin{equation}
\label{eq:quadratic-delta_app} 
\Delta = \frac{1}{2} \left\| \frac{1}{K} \sum_{k=1}^K (H_k-\bar H)d_k \right\|_{\bar H^{-1}}^2. 
\end{equation} 
Thus disagreement among the preferred local optima is insufficient for a selection-relevant population penalty. In the quadratic approximation, $\Delta$ arises from the interaction between their displacement and differences in local curvature. If all $H_k$ coincide, then $\bar\theta^*=\theta^*$ and $\Delta=0$ even when the local optima differ substantially. For non-quadratic risks, higher-order differences in local objective shape can play an analogous role.

\medskip \textbf{Homogeneous-limit behavior.} 
As $\alpha\to\infty$, the allocation probabilities converge to $\pi_k(u)=1/K$, so $P_k\to P$ for every $k$ and consequently $\theta_k^*\to\theta^*$ and $\Delta\to0$. To characterize the rate, write 
\[ 
\pi_k(u)=K^{-1}+\varepsilon_k(u), \qquad \sum_k\varepsilon_k(u)\equiv0, 
\] 
and let $r_\alpha$ denote the typical size of the perturbations. Then 
\[ 
\begin{array}{lll}
\Lambda_k = \Lambda+K\,G_k, \\[10pt] G_k(\theta) := \mathbb E_P[ \varepsilon_k(g(X))\ell(Z,\theta) ], \\[10pt]
\sum_kG_k\equiv0. 
\end{array}
\] 
Let 
\[ 
H:=\nabla^2\Lambda(\theta^*) 
\] 
and assume $H\succ0$. Under smoothness, perturbing the stationarity condition $\nabla\Lambda_k(\theta_k^*)=0$ gives \begin{equation}\label{eq:local-alpha-rate} \theta_k^*-\theta^* = -KH^{-1}\nabla G_k(\theta^*) + O(r_\alpha^2) = O(r_\alpha). \end{equation} Because \[ \sum_{k=1}^K\nabla G_k(\theta^*) = \nabla \left( \sum_{k=1}^K G_k \right)(\theta^*) = 0, \] the first-order perturbations cancel when the local optima are arithmetically averaged: \begin{equation}\label{eq:alpha-cancellation} \bar\theta^*-\theta^* = O(r_\alpha^2), \qquad \Delta = O(r_\alpha^4). \end{equation} Before Sinkhorn balancing, a component of a symmetric Dirichlet vector has variance \begin{equation} \operatorname{Var}[M_{s,k}] = \frac{K-1}{K^2(K\alpha+1)}. \end{equation} The final allocation probabilities after Sinkhorn balancing are not themselves Dirichlet distributed. If the balancing map is smooth in a neighborhood of the uniform allocation, the balanced perturbations retain the same asymptotic scale, giving 
\begin{equation}
\label{eq:alpha-rate} 
r_\alpha^2 = O\!\left((K\alpha+1)^{-1}\right) \qquad \text{as }\alpha\to\infty. 
\end{equation} 
Together with \eqref{eq:alpha-cancellation}, this explains the rapid attenuation of the selection-relevant population effect near homogeneous partitioning. The result is local to the regime $\alpha\to\infty$ and does not imply monotonicity of $\Delta$ in $\alpha$ or identify a universal optimal heterogeneity level.

\medskip \emph{Finite-sample connection.} Assume fixed $K$ and standard joint M-estimation regularity conditions for the balanced partition design. Each local empirical minimizer is root-$N$ consistent for its environment-specific population optimum, and therefore 
\begin{equation}
\label{eq:local-estimator-rate} 
\bar\theta-\bar\theta^* = O_p(N^{-1/2}). 
\end{equation} 
When $\Delta>0$, smoothness of the empirical and population risks then gives 
\begin{equation}
\label{eq:rate-alt} 
\Delta_N = \Delta+O_p(N^{-1/2}). 
\end{equation} 
Thus the empirical excess fluctuates around a non-vanishing population quantity. 
Under correct specification, all environments instead share the same population optimum $\theta^*$. Let 
\[ 
S_k:=\nabla L_k(\theta^*), \quad H_k:=\nabla^2\Lambda_k(\theta^*), \quad \bar H:=\frac{1}{K}\sum_{k=1}^K H_k. 
\] 
A first-order M-estimation expansion gives 
\begin{equation}
\label{eq:null-estimator-difference} 
\begin{array}{ll}
\bar\theta-\hat\theta_N &= -\frac{1}{K} \sum_{k=1}^K \left( H_k^{-1}-\bar H^{-1} \right)S_k + o_p(N^{-1/2})\\[10pt] 
&= O_p(N^{-1/2}). 
\end{array}
\end{equation}
Since $L_N$ is stationary at $\hat\theta_N$, the corresponding difference in loss is second order: 
\begin{equation} 
\Delta_N = L_N(\bar\theta)-L_N(\hat\theta_N) = O_p(N^{-1}). 
\end{equation} 
Equation~\eqref{eq:null-estimator-difference} also shows why the finite-sample discrepancy can be amplified when a local information matrix is poorly conditioned: the difference between averaged and centralized estimators involves the inverse local curvatures $H_k^{-1}$. This remains a finite-sample effect and vanishes with increasing $N$.
Combining the two regimes yields 
\begin{equation}
\label{eq:dichotomy_app} \Delta_N= 
\begin{cases} 
O_p(N^{-1}), & \text{well specified, identifiable},\\[4pt] 
\Delta+O_p(N^{-1/2}), & \text{when }\Delta>0. 
\end{cases}
\end{equation} 
The separation is between a sampling discrepancy that vanishes with sample size and a selection-relevant population effect that persists. 
\medskip

\medskip \emph{Surrogate informativeness.}
The conditional-invariance argument above requires only that the partitioning score
be a function of $x$. This condition guarantees validity of the construction, but it
does not guarantee that the induced environments are informative for DSCA.

To see what makes a surrogate informative, consider a small perturbation around the
pooled population law and write the environment weight as
\begin{equation} 
    K\pi_k(g(X)) = 1+\eta_k(X).
\end{equation}  
Let
\begin{equation} 
    q(X):=
    \mathbb{E}\!\left[
        \nabla_\theta\ell(Z,\theta^*)\mid X
    \right]
\end{equation}  
denote the conditional population loss gradient at the pooled optimum. A first-order
perturbation of the stationarity condition gives
\begin{equation} 
    \theta_k^*-\theta^*
    =
    -H^{-1}
    \mathbb{E}\!\left[\eta_k(X)q(X)\right]
    +O(\|\eta_k\|^2),
\end{equation}
where $H=\nabla^2\Lambda(\theta^*)$.
Hence the local optima move only to the extent that the covariate reweighting induced
by the surrogate is aligned with regions in which the candidate has a nonzero
conditional loss gradient.

For a correctly specified model, $q(X)=0$ pointwise, consistently with
Proposition~\ref{prop:correct_app}. For a misspecified model, a surrogate is informative
when its ordering separates regions of covariate space in which the preferred
calibration differs. Conversely, if the surrogate-induced allocation perturbations
are approximately orthogonal to $q(X)$, the first-order parameter displacement is
small and DSCA provides little additional selection pressure.

\emph{Worked examples.} Let $Y=X^2$ with $X\sim\mathrm{Unif}[0,1]$, and consider $K=2$ heterogeneous environments corresponding to $X_1=[0,1/2]$ and $X_2=(1/2,1]$. For the misspecified candidate $f(x,\theta)=\theta x$, 
\begin{equation} 
\theta_1^*=\frac{3}{8}, \quad \theta_2^*=\frac{45}{56}, \quad \bar\theta^*=\frac{33}{56}\approx0.5893, \quad \theta^*=\frac{3}{4}. 
\end{equation} 
Under squared-error loss,  
\begin{equation} 
\begin{array}{lll}
\Lambda(\theta^*) = \frac{1}{80} = 0.01250, \\[10pt]
\Lambda(\bar\theta^*) = \frac{331}{15680} \approx 0.02111,\\[10pt] \Delta = \frac{27}{3136} \approx 0.00861. 
\end{array}
\end{equation}  
Thus DSCA raises the population loss assigned to this misspecified structure from $0.01250$ to approximately $0.02111$, corresponding to $\Delta\approx0.00861$, or an increase of about $69\%$ relative to centralized calibration. Under homogeneous environments the local and pooled optima coincide and $\Delta=0$. 

The second example shows why misspecification or disagreement alone is insufficient. For the misspecified constant candidate $f(x,\theta)=\theta$ under the same two environments, 
\begin{equation} 
\theta_1^*=\frac{1}{12}, \qquad \theta_2^*=\frac{7}{12}, \qquad \bar\theta^*=\frac{1}{3}=\theta^*. 
\end{equation}  
Despite the substantial disagreement between the two local pseudo-true parameters, their arithmetic mean coincides exactly with the centralized optimum, and therefore $\Delta=0$. Here $H_1=H_2=2$, so this cancellation is precisely the behavior predicted by \eqref{eq:quadratic-delta_app}. DSCA introduces no population-level selective penalty for this candidate.

\medskip \emph{Scope and limitations.} The analysis above assumes identifiable population optima, smooth local risks, fixed $K$, and an idealized calibration step based on exact empirical minimizers. Symbolic simplification reduces exact algebraic non-identifiability before calibration, but weak identification, near-singularity, and numerical optimization error may remain in practice. Strong heterogeneity can amplify these effects by reducing the information available within individual partitions, whereas the homogeneous limit suppresses the population DSCA effect. These limitations motivate the empirical sensitivity analysis with respect to $\alpha$ and the optimization controls reported in the main text.


\section{Per-Function Results on the Synthetic Track}

Figures~\ref{fig:synth_100} and~\ref{fig:synth_200} report the success rates on each synthetic function for the 100-sample and 200-sample settings, respectively. Within each figure, the top row shows the low noise setting and the bottom row the high noise setting. Each configuration has been tested with 10 different seeds, and an equation is considered recovered if it achieves $R^2 > 0.99$ on the noiseless test set. These results underlie the win/tie/loss counts and the aggregate counts of the main text.

All evaluated methods struggle significantly with $f_3$, frequently failing to identify the ground-truth model. Success rates remain at or near 0\% across almost all configurations. The only exceptions are DSCA and BFGS reaching 10\% under BIC in the 200-sample high-noise regime. This difficulty likely stems from the sine term, which usually represents a significant challenge for SR algorithms.

For $f_1$ and $f_2$, DSCA generally outperforms BFGS and LM, particularly under the NLL and AIC selection criteria. For instance, on $f_1$ in the 200-sample low-noise regime, DSCA achieves an 80\% success rate using NLL, compared to 50\% for BFGS and 30\% for LM. Performance is instead comparable across the three optimizers on $f_4$ and $f_5$. There, success rates typically cluster between 60\% and 70\% under low noise with 200 samples, and the few losses of DSCA reported in Table~IV occur.

These comparative trends are consistent across both noise regimes. The higher noise setting yields lower overall recall, a degradation most pronounced for $f_4$ and $f_5$. For example, in the 200-sample regime, high noise drops the success rates on $f_4$ from approximately 70\% to between 0\% and 20\% for most criteria.

\begin{figure*}
    \centering
    \includegraphics[width=0.95\linewidth]{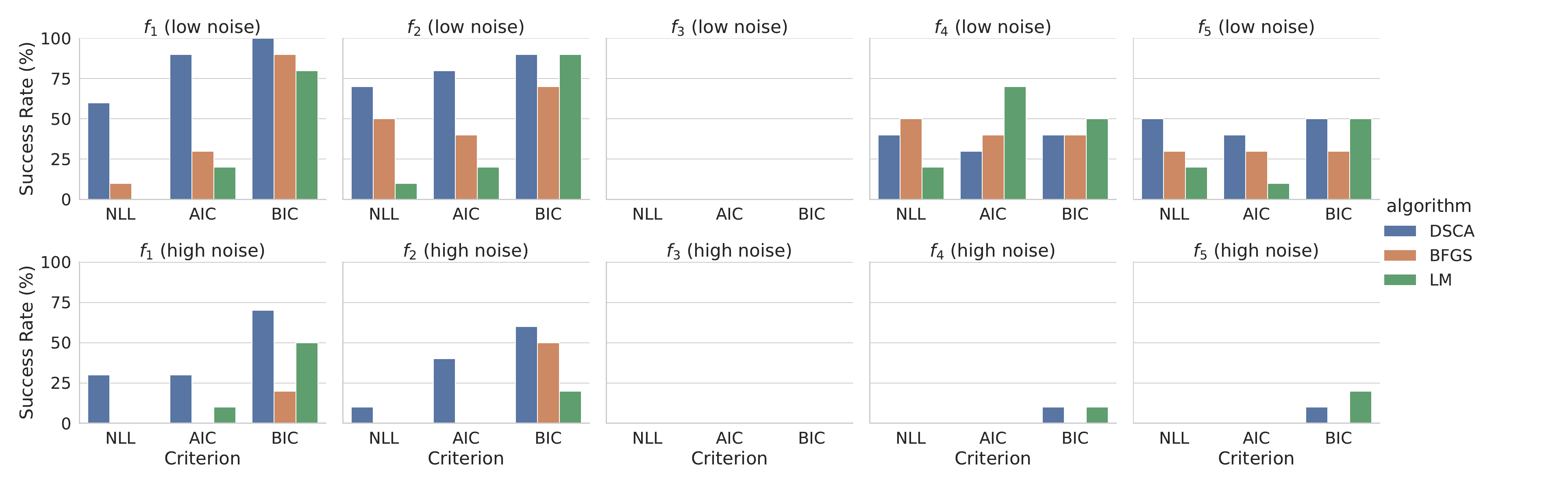}
    \caption{Success rates on synthetic datasets ($N=100$) under low and high noise regimes.}
    \label{fig:synth_100}
\end{figure*}

\begin{figure*}
    \centering
    \includegraphics[width=0.95\linewidth]{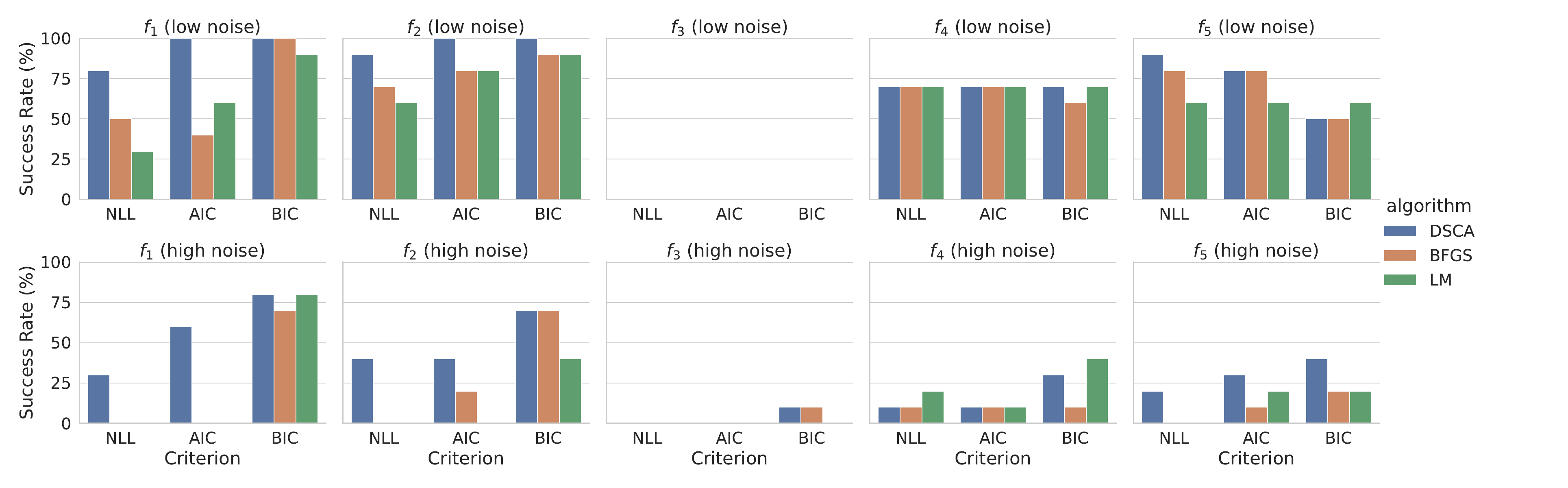}
    \caption{Success rates on synthetic datasets ($N=200$) under low and high noise regimes.}
    \label{fig:synth_200}
\end{figure*}

\end{document}